\documentclass[11pt]{article}
\usepackage[dvips]{graphics}
\usepackage{epsfig}
\usepackage{multirow}
\usepackage{array}
\usepackage{extarrows}
\usepackage{graphicx}
\usepackage[section]{placeins}
\usepackage{relsize}
\usepackage{comment}
\usepackage[titletoc,title]{appendix}
\numberwithin{equation}{section}
\usepackage{amsthm}
\usepackage{enumerate}
\usepackage[ruled,vlined,linesnumbered,noresetcount]{algorithm2e}
\usepackage{nomencl}
\makenomenclature

\usepackage{natbib}
\usepackage{appendix}
\usepackage{amssymb,amsmath}

\DeclareMathOperator*{\argmin}{arg\,min}
\usepackage{latexsym}
\usepackage[usenames]{color}
\usepackage{pgf}
\usepackage{subfig}
\numberwithin{equation}{section}

\newcommand{\be}{\begin{equation}}
\newcommand{\ee}{\end{equation}}
\newcommand{\beaa}{\begin{eqnarray*}}
\newcommand{\eeaa}{\end{eqnarray*}}
\newcommand{\bea}{\begin{eqnarray}}
\newcommand{\eea}{\end{eqnarray}}

\newcommand{\bei}{\begin{itemize}}
\newcommand{\eei}{\end{itemize}}

\def\Tr{\mathrm{Tr}}

\newtheorem{theorem}{ \noindent T{\footnotesize HEOREM}}
\newtheorem{prop}{ \noindent P{\footnotesize ROPOSITION}}[section]
\newtheorem{lemma}{ \noindent L{\footnotesize EMMA}}[section]
\newtheorem{coro}{ \noindent C{\footnotesize OROLLARY}}
\newtheorem{remark}{ \noindent R{\footnotesize EMARK}}[section]

\newtheorem{assumption}{Assumption}
\newtheorem{example}{Example}

\renewcommand{\nompreamble}{The list describes several symbols that will be later used within the body of the document}

\begin{document}

	\title{Non Asymptotic Analysis of Online Multiplicative Stochastic Gradient Descent}
\author{Riddhiman Bhattacharya\thanks{Purdue University, bhatta76@purdue.edu} \quad \quad Tiefeng Jiang\thanks{University of Minnesota, jiang040@umn.edu}}

\date{}
\maketitle

\footnotetext[1]{Krannert School of Management, Purdue University, 403 W State St., West Lafayette, IN47907, USA, bhatta76@purdue.edu. %\newline  \indent \  \
.}
\footnotetext[2]{Dr. Tiefeng Jiang  is partly supported by NSF grant DMS-1916014. %\newline  \indent \  \
}

\begin{abstract}
	\noindent Past research has indicated that the covariance of the Stochastic Gradient Descent (SGD) error done via minibatching plays a critical role in determining its regularization and escape from low potential points. Motivated by some new research in this area, we prove universality results by showing that noise classes that have the same mean and covariance structure of SGD via minibatching have similar properties. We mainly consider the Multiplicative Stochastic Gradient Descent (M-SGD) algorithm as introduced in previous work~\cite{wu2020noisy}, which has a much more general noise class than the SGD algorithm done via minibatching. We establish non asymptotic  bounds for the M-SGD algorithm in the Wasserstein distance. We also show that the M-SGD error is approximately a scaled Gaussian distribution with mean $0$ at any fixed point of the M-SGD algorithm.
\end{abstract}
\noindent \textbf{Keywords:\/} Multiplicative Stochastic Gradient Descent, Stochastic Gradient Descent, diffusion, Wasserstein distance, optimization, Central Limit Theorem

\section{Introduction}\label{introduction}
SGD is traditionally focused on finding the minimum value of a function (also called objective function in optimization literature) taken into consideration due to our given scientific problem~\cite{mertikopoulos2020almost,zhou2019sgd,jin2019nonconvex,ghadimi2013stochastic,robbins1951stochastic}. It has proven to be an extremely effective method in tackling hard problems in multiple fields such as machine learning~\cite{bottou2018optimization}, statistics~\cite{toulis2017asymptotic}, electrical engineering~\cite{golmant2018computational}, etc.

The SGD algorithm and some of its properties are now well known especially in the context of machine learning. The iterative version of the SGD algorithm is given as
\begin{align}\label{algo1}
x_{k+1}=x_k -\gamma_k \nabla g(x_k)-\gamma_k \xi(x_k)_{k+1}, \ \text{for } k=0,1,2,\cdots
\end{align}
where $\gamma_k$ is called the step size (it is possible to choose $\gamma_k=\gamma$ for all $k$), $g(\cdot)$ is the objective function (function for which we want the optimum) and $\xi(x_k)_{k+1}$ is the error term that may or may not depend on the current point $x_k$.  The SGD algorithm can be thought of as a stochastic generalization of the Gradient Descent (GD) algorithm which is one of the oldest algorithms for optimization. One of the most popular methods for performing SGD in practice is to perform SGD via minibatching ~\cite{bottou1991stochastic}.
From this point onward, we refer to this algorithm as minibatch SGD.

Though initially proposed to remedy the computational problem of Gradient Descent, recent studies have shown that minibatch SGD has the property of inducing an implicit regularization which prevents the over parametrized models from converging to the minima~\cite{zhang2021understanding,hoffer2017train,keskar2016large}. This implies that according to empirical findings, smaller batch size minibatch SGD improves model accuracy for problems with flat minima while large batch size minibatch SGD improves model accuracy for problems with sharper minima.  This phenomenon lead to an investigation~\cite{wu2020noisy} which serves as the inspiration for our research. In this paper the authors introduced the Multiplicative Stochastic Gradient Descent (M-SGD) algorithm which is stated below (Algorithm~\ref{algo3}). 

The term $\mathcal{V}$ as defined in Algorithm~\ref{algo3} is the same as $\xi(x_k)_{k+1}$ in equation~\ref{algo1}. The $\frac{1}{n}$ term is added to $\mathcal{V}$ in Algorithm~\ref{algo3} to give us the gradient term in the SGD equation. This algorithm has been extensively applied by in the context of deep learning~\cite{wu2020noisy}. Previous work exhibits that for certain problems the M-SGD algorithm gives good performance and provide simulation results using CIFAR data. The authors mainly consider the problem of online linear regression and exhibit the convergence of the M-SGD algorithm on average to the optimum point. The authors also consider certain noise classes for $\mathcal{V}$ and establish theoretical results.\\

\begin{algorithm}[H]\label{algo3}
	\setcounter{AlgoLine}{0}
	\textbf{Input}: Initial parameter $\theta_0 \in \mathbb{R}^p$, training data $\left\{\left(x_i,y_i\right)\right\}_{i=1}^{n}$, loss function $l_i(\theta)=l((x_i,y_i),\theta)\in \mathbb{R}$ and step size (also called learning rate) as $\gamma>0$.\\
	\textbf{for} $k=0,1,\ldots, K-1$ \textbf{do}\\ 
	Generate the sampling noise $\mathcal{V}\in \mathbb{R}^n$.\\
	Compute the sampling vector $W=\frac{1}{n}l+\mathcal{V}$, where $l$ is the vector with all entries $1$. \\
	Update the randomized loss as $L(\theta_k)=\mathcal{L}(\theta_k)W$, where $\mathcal{L}(\theta_k)=\left[l_1(\theta_k),l_2(\theta_k),\cdots,l_n(\theta_k)\right]$. \\
	Update the parameter as $\theta_{k+1}=\theta_k-\gamma \nabla L(\theta_k)$.\\
	\textbf{end for}\\
	\textbf{Output} $y_K$.
	\caption{Multiplicative Stochastic Gradient Descent (M-SGD)}
\end{algorithm}
\medskip 	 	 
Motivated by the applications of the M-SGD algorithm in deep learning, we study the properties of the online version of the M-SGD algorithm. Another motivation is that, the general SGD algorithm has shown to perform with considerable success for convex optimization problems and with some success for nonconvex optimization problems~\cite{polyak1992acceleration,bottou2018optimization,ghadimi2013stochastic}.  Due to the stochastic nature of the algorithm it is believed that the SGD escapes low potential points~\cite{hu2018diffusion} and thus has the scope to visit multiple of them. This is extremely similar to sampling and is more observed in SGD with fixed step size, i.e., $\gamma_k=\gamma$ for all $k$. There has been enormous amount of work for SGD with both fixed step size and variable step size, with different error structures~\cite{dieuleveut2017bridging,leluc2020towards,yu2020analysis}.

We study an online version of the M-SGD algorithm in which at each stage of the iteration we generate random training data $u_1,u_2,\cdots,u_n \overset{iid}{\sim}Q$, where $Q$ is some probability measure such that loss function defined as  $l(\theta,u)$ is an unbiased estimator of our objective function $g(\theta)$. Also, at each stage, we generate the weight vector $W$ as expressed in Algorithm~\ref{algo3}. 
Throughout our discussion we consider $p$ to be the dimension of our problem. Our training data are random vectors in $\mathbb{R}^p$ where one can think of $u_1,u_2,\cdots,u_n$ as being generated from some ``true" parameter in the parameter space or if we already have fixed training data $x_1,x_2,\cdots,x_t$ then we may think of $u_i$ as being one of the training data values with uniform probability, i.e., $u_i=x_j \ \text{with probability} \ \frac{1}{t} \ \text{for all }i,j$. Thus our set-up helps us address a large class of problems and permits a level of flexibility. Our iterative algorithm is as follows
\begin{align}\label{iteration1}
y_{k+1}=y_k-\gamma \left(\sum_{i=1}^{n} w_{i,k}\nabla l(y_k,u_{i,k})\right),
\text{ for }k=0,1,2,\ldots
\end{align} 
where $\gamma$ is the step size and $l$ is the loss function.
Note that we write $w_{i,k},u_{i,k}$ instead of $w_i$ and $u_i$. This is because we refresh both the random training data and the ``weights" at each step of the iteration. Hence this is an online version of the M-SGD algorithm as in we use new data at each iteration.	
\section{Preliminaries}\label{Preliminaries}
\paragraph*{}We consider an objective function $g(\cdot)$ and a loss function $l(\cdot,\cdot)$. We have sample training data as $u_1,u_2,\cdots,u_n$ as iid from some distribution $Q$. Our choice of $Q$ should be such that $\mathbb{E}\left( l(\theta,u_i)\right)=g(\theta)$ for all $\theta \in \mathbb{R}^p$. We consider $|\cdot|$ to be the Euclidean norm (note that on $\mathbb{R}^1$ this is just the absolute value) and denote the transpose of any matrix $A$ as $A^{\mathsf{T}}$. We denote by $X \sim (a,b)$ as a random variable with $\mathbb{E}(X)=a$ and $Var(X)=b$.

We consider the following algorithms which shall appear in our analysis repeatedly.
\begin{align}
\label{sgd}x_{k+1}&=x_k-\gamma \nabla g(x_k)+\frac{\gamma}{\sqrt m} \sigma(x_k)\xi_{k+1},\\
\label{msgd} x_{n,k+1}&=x_{n,k}-\gamma \sum_{i=1}^{n}w_{i,k}\nabla l(x_{n,k},u_{i,k}),
\end{align}
which are respectively the gradient  the SGD algorithm with scaled Gaussian error~\eqref{sgd} and the M-SGD algorithm~\eqref{msgd}. Here $k=0,1,2,\cdots,K$ and $\gamma$ is the learning rate or the step size. We also consider the following stochastic processes
\begin{align}
\label{intsgd} D_t&=D_{k\gamma}-\left(t-k\gamma\right)\nabla g\left(D_{k\gamma}\right)+\sqrt{\frac{\gamma}{m}}\sigma\left(D_{k\gamma}\right)\left(B_{t}-B_{k\gamma}\right),\\
\label{sgddiff} dX_t&=-\nabla g(X_t)+\sqrt{\frac{\gamma}{m}}\sigma(X_t)dB_t ,\\
\label{intmsgd} Y_{n,t}&=Y_{n,k\gamma}- \left(t-k\gamma\right)\sum_{i=1}^{n}w_{i,k}\nabla l(Y_{n,k\gamma},u_{i,k}).
\end{align}
In all these equations $t\in (k\gamma,(k+1)\gamma]$ with $k\le K$ and $K\gamma=T$, where $T$ is the time horizon of all the stochastic processes as defined throughout the work. All the initial points are fixed at a single point denoted by $x_0$. The equations  \eqref{intsgd} and \eqref{intmsgd} are continuous versions of~\eqref{sgd} and~\eqref{msgd} respectively. 
Our claim is that M-SGD is close to the diffusion as described by \eqref{sgddiff}. By ``close to", we imply the Wasserstein-2 distance function between two probability measures. Recall the definition of the Wasserstein-2 distance as 
\begin{align}\label{wassdef}
W_2(\mu,\nu)=\inf_{\tilde Q\in \Gamma(\mu,\nu)}\left(\int_{X\times Y}\left|x-y\right|^2 d\tilde Q\right)^{1/2}.
\end{align}
Here $\Gamma(\mu,\nu)$ denotes the set of all couplings of the probability measures $\mu,\nu$. Note that since random variables generate a probability measure on the real line or $\mathbb{R}^p$ for any $p$, we can also define the Wasserstein-2 distance with respect to random variables in analogous fashion. 

We now state our assumptions.
\begin{assumption}\label{assm1}
	All starting points of equations~\eqref{sgd}-\eqref{intmsgd} and~\eqref{gd}, \eqref{gdcont} is the same fixed quantity and is denoted by $x_0$.
\end{assumption} 
\begin{assumption}\label{assm2}
	The loss function $l(\theta,x)$ is continuously twice differentiable in $\theta$ for each $x \in \mathbb{R}^p$. Also for all $u \in \mathbb{R}^p$ there exists $h_1(u)>0$ such that for all $\theta_1, \theta_2 \in \mathbb{R}^p$,
	\begin{align}
	\left|\nabla l(\theta_1, u)-\nabla l(\theta_2, u)\right| \le h_1(u)\left|\theta_1 - \theta_2\right|.
	\end{align}
	In addition, there exists $\theta_0$ such that $\nabla l(\theta_0, u)$ is $L^3(\Omega,P)$ and $h_1(u)$  has finite  moment generating function with $\mathbb{E}\left[\exp(\tilde\gamma h_1(U))\right]<\infty $ for some $\tilde\gamma>0$.				
\end{assumption}
\begin{remark}\label{rem1}
	Note that the first part of assumption~\ref{assm2} implies that $\mathbb{E}\left|\nabla l(\theta,u)\right|^3<\infty$ for all $\theta$. 
	This is easily seen since $\left|\nabla l(\theta,u)\right|^3\le 4(|\nabla l(\theta_0,u)|^3+h_1^3(u)|\theta-\theta_0|^3)$.
\end{remark}
\begin{remark}\label{ass21}
	Note that assumptions~\ref{assm2} implies that for all $\theta$, we have
	\begin{align*}
	\mathbb{E}\left(\nabla l(\theta,u)\right)=\nabla g(\theta).
	\end{align*}			
\end{remark} 
Indeed one can see this by using the dominated convergence theorem (DCT). A detailed explanation of this is provided in Lemma~\ref{lemmamain} in the  Appendix.

\begin{remark}\label{ass22}
	Note that assumptions~\ref{assm2} also implies
	\begin{align*}
	\left|\nabla g(\theta_1)-\nabla g(\theta_2)\right|\le L\left|\theta_1-\theta_2\right|,
	\end{align*}
	for all $\theta$,  where $L=\mathbb{E}[h(u)]$.			
\end{remark} 
Remark~\ref{ass22} is a standard assumption in the optimization literature.				

\begin{assumption}\label{assm3}
	There exists $L_1>0$ such that 
	\begin{align*}
	\left|\left|\sigma(\theta_1)-\sigma(\theta_2)\right|\right|_2 \le L_1 \left|\theta_1-\theta_2\right|,
	\end{align*}  where $||\cdot||_2$ is the spectral norm of the operator and $\sigma(\cdot)=Var(\nabla l(\cdot,u))$.\\
\end{assumption}

\begin{remark}\label{assm31}
	Note that assumption~\ref{assm3} also implies 
	\begin{align*}
	\left|\left|\sigma(\theta_1)-\sigma(\theta_2)\right|\right|_F \le \sqrt{p}L_1 \left|\theta_1-\theta_2\right|,	
	\end{align*}
	where $||\cdot||_F$ denotes the Frobenius norm of a matrix. 
\end{remark}
This is easy to see since
$
\left|\left|A\right|\right|^2_F=\sum_{i=1}^{p} e^{\mathsf{T}}_iA^{\mathsf{T}}Ae_i
\le \sum_{i=1}^{p} \left|Ae_i\right|^2
\le \sum_{i=1}^{p} \left|\left|A\right|\right|^2_2
=p \ \left|\left|A\right|\right|^2_2.
$
This assumption implies that the covariance structure for the randomness in the training data has some level of linear control. 
\begin{remark}
	All our results are valid for any $\sigma(\theta)$ such that $\sigma(\theta)\sigma(\theta)^{\mathsf{T}}=\sigma^2(\theta)$ with $\sigma(\theta)$ as Lipschitz in the spectral norm. For the sake of convenience we assume $\sigma(\theta)=\sigma^2(\theta)^{1/2}$ throughout our work.
\end{remark}

\begin{assumption}\label{assm4}
	Given any $n$, the weight vectors at each iteration in equation~\eqref{msgd} are iid \space  $W=(w_1,w_2,w_3,\cdots,w_n)^{\mathsf{T}}$ where $W$ is any random vector with $\mathbb{E}(W)=1/n \,(1,1,\cdots,1)^{\mathsf{T}}$ and the variance covariance matrix is $\Sigma$, i.e., 
	\begin{align*}
	W \sim \left(\frac{1}{n}(1,1,\cdots,1)^{\mathsf{T}},\Sigma\right),
	\end{align*} where $\left(\Sigma\right)_{i,i}=\sigma_{ii}=\frac{n-m}{mn^2}$ and $\left(\Sigma\right)_{i,j}=\sigma_{ij}=-\frac{n-m}{mn^2(n-1)}$.	
\end{assumption}

An immediate example of such a case is minibatch SGD, which is the most widely used SGD algorithm in practice. In this case the $w_{i,k}=1/m$ is it is included in the sample and is 0 otherwise. Here $m$ denotes the minibatch size. This is the hypergeometric set-up and it is not hard to show that the mean and the variance of the weights are as advertised above in assumption~\ref{assm4}. Indeed it is easy to see that $\mathbb{E}(w_i)=1/n$,
$Cov(w_i,w_j) =-(n-m)/mn^2(n-1)$
and  $Var(w_i)=(n-m)/mn^2$.
This gives an intuition for the mean vector and the covariance matrix in the general case. One point to note is that the variance matrix for $W$ is not strictly positive which implies that $W$ lies in a lower dimensional space. Indeed, this is true as one has $\sum_{i=1}^{n}w_{i,k}=1$ almost surely for all $k$ which is evident from the definition of $w_{i,k}$.
\begin{assumption}\label{assm5}
	At each iteration step ($k$) in equation~\eqref{msgd}, $u_{i,k}$ are generated iid $Q$, i.e.,
	\begin{align*}
	u_{i,k} \overset{iid}{\sim} \ Q , \quad \text{for all } i=0,1,2,\cdots,n \ \text{and } k=0,1,2\cdots,K
	\end{align*} where $Q$ denotes the probability measure such that $\mathbb{E}( l(\theta,u_{i,k}))=g(\theta)$ for all $\theta$, $i$ and $k$.
\end{assumption}
This assumption on the training data implies that we have data of similar type coming into consideration at each time point. However, it is our strong belief that most of our results still hold even when we do not refresh our training data at each iteration step.

Next we make further assumptions on $W$ which enable us in proving the Gaussian nature of the M-SGD error. Define $l=(1,1,\cdots,1)^{\mathsf{T}}$, i.e., the vector of all entries 1. Note that $$\sqrt{\frac{n-m}{mn(n-1)}}\left(I-\frac{1}{n}ll^{\mathsf{T}}\right)=\Sigma^{1/2},$$ where $\Sigma$ is defined in assumption~\ref{assm4}. 
\begin{assumption}\label{assm6case1}
	For $W$ as defined in assumption~\ref{assm4}, we further have
	\begin{align}\label{1stform}
	W=\sqrt{\frac{n-m}{mn(n-1)}}\left(I-\frac{1}{n}ll^{\mathsf{T}}\right)X+\frac{1}{n}l,
	\end{align}
	where $X=(X_1,X_2,\cdots,X_n)$ and $X_1,X_2,\cdots,X_n$ are iid sub-Gaussian  with mean $\mu$ and variance $1$.
\end{assumption}
This assumption enables us to consider a large class of distributions for $W$. Note that we would require higher moment conditions to prove CLT in any case due to the dependent structure of $W$. This assumption addressed this and covers a large class of distributions.

Next we state a second assumption on the $W$ vector which includes the case of the weights being non negative.

\begin{assumption}\label{assm6case2}
For the random variables $w_i$ as defined in assumption~\ref{assm4}, we have the following conditions
\begin{itemize}
    \item $w_i$ are exchangeable.
    \item $w_i \ge 0$ for all $i$ with $\sum_{i=1}^{n} w_i=1$.
    \item $\max_{1\le i\le n} \sqrt{m}\left|w_i-1/n\right|\overset{P}{\rightarrow}0$ as $n\to \infty$.
    \item $m \sum_{i=1}^{n}(w_i-\frac{1}{n})^2 \overset{P}{\rightarrow} 1$ as $n \to \infty$.
\end{itemize}
\end{assumption}
$ $\linebreak
Note that we need two separate assumptions as in the first assumption the structure of $W$ forces $W$ to have negative values in the support for each coordinate. This is addressed in Assumption~\ref{assm6case2}.

In the later sections, we use the assumption of strong convexity on $g$. The assumption is as such
\begin{assumption}\label{assm7}
	A function $g$ is $\lambda$-strongly convex if $\lambda I \le\nabla^2 g (\theta)$ for some $\lambda>0$ and all $\theta \in \mathbb{R}^p$. 
\end{assumption}
\begin{remark}\label{assm71}
	Note that this also implies $g(x)\ge g(y)+(x-y)^{\mathsf{T}}\nabla g(y)+\frac{\lambda}{2}|x-y|^2$ for all $x,y \in \mathbb{R}^p$.	
\end{remark}
Note that the assumption of strong convexity indeed forces $g(\cdot)$ to have a minima. In fact, we can say that there exits a unique $x^*$ such that $\inf_{x} g(x)=g(x^*)$.

\section{Main Results}\label{mainresults}
Throughout this chapter, we denote $g(\cdot)$ to be the objective function and $l(\cdot,\cdot)$ as the loss function. We consider the problem in the regime where $T=\gamma K$, where $T$ is fixed. As mentioned in Algorithm~\ref{algo3}, $\gamma$ is the step size and $K$ is the number of iterations. We shall consider $0<\gamma<1$ throughout our work. 
We shall refer to $n$ as the sample size and $m$ as the minibatch size with  $n\to \infty$, $m\to \infty$, $m/n\to \gamma^*$ and $0\le \gamma^*\le 1$. The following is the algorithmic representation of the online M-SGD algorithm\\

\begin{algorithm}[H]\label{algo4}
	\setcounter{AlgoLine}{0}
	\textbf{Input}: Staring point $x_0 \in \mathbb{R}^p$, loss function $l((x,u))\in \mathbb{R}$ and step size (also called learning rate) as $1>\gamma>0$.\\
	\textbf{for} $k=0,1,\cdots, K-1$ \textbf{do}\\ 
	Generate random training data for the $k$-th iteration $u_{1,k},u_{2,k},\cdots,u_{n,k}$.\\
	Generate the weight vector $W_k\in \mathbb{R}^n$, for the $k$-th iteration; where $W_k=(w_{1,k},w_{2,k},\cdots,w_{n,k})^{\mathsf{T}}$, as per assumption~\ref{assm4}.\\
	Update to the next step as $x_{k+1}=x_k-\gamma \sum_{i=1}^{n}w_{i,k} \nabla l(x_k,u_{i,k})$.\\
	\textbf{end for}\\
	\textbf{Output} $x_K$.
	\caption{Online Multiplicative Stochastic Gradient Descent (Online M-SGD)}
\end{algorithm}
\subsection{The Central Limit Theorem of M-SGD Error}
Note that the update step in Algorithm~\ref{algo4} can also be expressed as 
\begin{align*}
x_{k+1}=x_k-\gamma \nabla g(x_k)-\gamma \sum_{i=1}^{n}w_{i,k}\left(\nabla l(x_k,u_{i,k})-\nabla g(x_k)\right)
\end{align*}
owing to the fact $\sum_{i=1}^{n} w_{i,k}=1$, which follows from assumption~\ref{assm4}. The term $$\sum_{i=1}^{n}w_{i,k}\left(\nabla l(x_k,u_{i,k})-\nabla g(x_k)\right)$$ is called the M-SGD error. We first exhibit that for any $\theta$ a scaled version of this error is approximately Gaussian or for all $\theta \in \mathbb{R}^p$, we have
\begin{align*}
\sqrt{m}\sum_{i=1}^{n}w_{i,k}\left(\nabla l(\theta,u_{i,k})-\nabla g(\theta)\right) \overset{d}{\approx} N(0,\sigma^2(\theta)).
\end{align*}
The symbol $\overset{d}{\approx}$ implies has approximate distribution. This approximation holds when $n$ is large. To put it more rigorously, we have \[\sqrt m \sum_{i=1}^{n} w_{i,k}\left(\nabla l(\theta,u_{i,k})-\nabla g(\theta)\right) \overset{d}{\to} N(0,\sigma^2(\theta)) \ \text{as }n\to \infty.\] One can find similar work in Bootstrap literature~\cite{arenal1996zero}. However, the problem considered in such cases is somewhat different from ours and thus this is a new way of looking at the SGD error. 

We claim that if $W_k$ has mean and covariance structure as given in assumption~\ref{assm4}, some key properties of the minibatch SGD are retained.  This in some way is a universality of the weights. Based on the relation between $m,n$ we divide the first problem into cases ($\gamma^*=1$ and $\gamma^*<1$) and analyze the Gaussian nature of the error term in Algorithm~\ref{algo4}. For ease of notation, we shall write $u_{i,k}=u_i$ for this section as the Gaussian property is independent of $k$. 

We invoke assumptions~\ref{assm4}, \ref{assm5} and~\ref{assm6case1} or \ref{assm5} and~\ref{assm6case2} to help us in obtaining the following CLT.

\begin{theorem}\label{cltthm2}
	We consider the regime $m/n \to \gamma^* \text{;} \quad 0\le \gamma^* < 1$ as $n\to \infty, \ m\to \infty$, with $w=(w_1,w_2,\dots,w_n)$ as defined in assumption~\ref{assm4}. We also take the assumptions~\ref{assm5} and~\ref{assm6case1} to hold true. In such a setting, we have
	$$\sqrt{m}(\sum_{i=1}^{n}w_i \nabla l(\theta,u_i)-\nabla g(\theta)) \xrightarrow{d} N(0,\sigma^2(\theta))$$
	as $n\to \infty$.
\end{theorem}
The proof of this theorem is given in the Appendix.
\begin{remark}\label{remarkclt1}
	The CLT still holds if we consider iid mean zero random variables with finite third moment instead of $(\nabla l(\theta,u_i)-\nabla g(\theta))$. The asymptotic variance is dependent on the distribution of the iid random variables.
\end{remark}
The proof of Theorem~\ref{cltthm2} provides evidence for remark~\ref{remarkclt1}.
\begin{theorem}\label{positive:wts:thm}
    Let assumptions~\ref{assm4} and~\ref{assm6case2} hold. Also, let $m/n\to 0$ as $n\to \infty$. Then for any $\theta \in \mathbb{R}^p$, we have
    \begin{align*}
        \sqrt{m}(\sum_{i=1}^{n}w_i \nabla l(\theta,u_i)-\nabla g(\theta)) \xrightarrow{d} N(0,\sigma^2(\theta))
    \end{align*}
    as $n \to \infty$.
\end{theorem}
The proof of Theorem~\ref{positive:wts:thm} is provided in the Appendix.

\begin{example}[Example for assumption~\ref{assm6case1}]
An example of weights where this structure is observed is when $W\sim N(\mu,\Sigma)$. In this case it is easy to see that 
$$W=\Sigma^{1/2}X+\mu,$$ where $X\sim N(0,I)$. We can easily check that the assumptions for assumption~\ref{assm6case1} are satisfied here.
We consider $U_i=(U_{i1},U_{i2},\cdots,U_{i6})^{\mathsf{T}}$ where have each $U_{ij}\sim Unif(-1,1)$ iid as our data. We have $W\sim N(\mu,\Sigma)$ where $\mu$ and $\Sigma$ are as defined in assumption~\ref{assm4}. The dimension of the problem is taken to be $6$, i.e., $p=6$ with $n=10^4$, $m=2000$ and $10^3$ samples of $\sqrt{m}\sum_{i=1}^{n}w_{i,k}U_i$ are generated. We observe the distribution of the resultant data. The histogram of the one dimensional projection of the data along the standard basis is presented in Figure~\ref{fig2}.
\end{example}

\begin{figure}[!htb]
	\includegraphics[width=15cm,height=7cm]{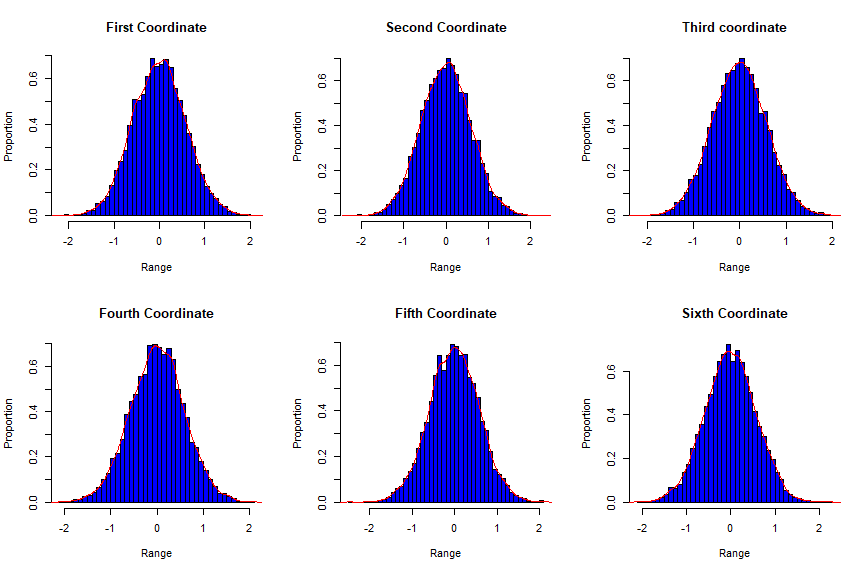}
	\caption{Histogram of the 10000 samples of $\sqrt{m}\sum_{i=1}^{n}w_{i,k}U_i$. Here we have $p=6$ with $n=10^4$, $m=2000$. The weight vector $W=(w_1,w_2,\cdots,w_n)^{\mathsf{T}}$ is distributed as per $N(\mu, \Sigma)$ where $\mu$ and $\Sigma$ are as specified in assumption~\ref{assm4}.}
	\label{fig2}
\end{figure}

\begin{example}[Example for assumption~\ref{assm6case2}]
     The simplest example for the positive case is the minibatch/hypergeometric random variables where $w_i=1/m$ if the $i$-th element is selected and $w_i=0$ otherwise. Also $\sum_{i=1}^{n}w_i=1$ which implies exactly $m$ indices are selected out of $n$. In this case it is easy to verify that the assumptions~\ref{assm4} and assumption~\ref{assm6case2} hold. However, we provide a more non-trivial example which is the Dirichlet Distribution.	Consider a vector \[w\sim Dir\left(\left(\frac{m-1}{n-m},\frac{m-1}{n-m},\cdots,\frac{m-1}{n-m}\right)_{1\times n}\right).\]

 Note that a $w$ which follows the given Dirichlet distribution has the property that $w_i$ are exchangeable, $w_i \ge 0$ and $\sum_{i=1}^{n} w_i=1$. Also, some minor calculations will show that $$w=(w_1,w_2,\cdots,w_n)^{\mathsf{T}} \sim \left(1/n(1,1,1,\cdots,1)^{\mathsf{T}},\Sigma\right).$$
Here $\Sigma$ is defined as previously. 
\begin{prop}\label{Dirwtlemma}
	If $w\sim Dir((\frac{m-1}{n-m},\frac{m-1}{n-m},\cdots,\frac{m-1}{n-m}))$ with $m/n \to 0$ as $n\to \infty$, then     \begin{align*}
        \sqrt{m}(\sum_{i=1}^{n}w_i \nabla l(\theta,u_i)-\nabla g(\theta)) \xrightarrow{d} N(0,\sigma^2(\theta))
    \end{align*}
    as $n \to \infty$.
\end{prop}
The proof of Proposition~\ref{Dirwtlemma} is provided in the Appendix.

 As a numerical example we consider $U_i\sim Unif(-1,1)$ for all $i$. Note that we take dimension of this problem as $1$, i.e., $p=1$ with $n=10^4$, $m=2000$. We generate $10^4$ samples of $\sqrt{m}\sum_{i=1}^{n}w_{i}U_i$. In this example the weight vector $W=(w_1,w_2,\cdots,w_n)^{\mathsf{T}}$ is simulated from $Dir((\frac{1999}{8000},\frac{1999}{8000},\cdots,\frac{1999}{8000}))$ which is the Dirichlet distribution with parameter vector of length $10^4$, given as $(\frac{1999}{8000},\frac{1999}{8000},\cdots,\frac{1999}{8000})^{\mathsf{T}}$. The results is exhibited in Figure~\ref{fig1}. From the plot it seems that the samples are distributed as per the normal distribution. 
\begin{figure}[!htb]
	\includegraphics[width=15cm,height=5cm]{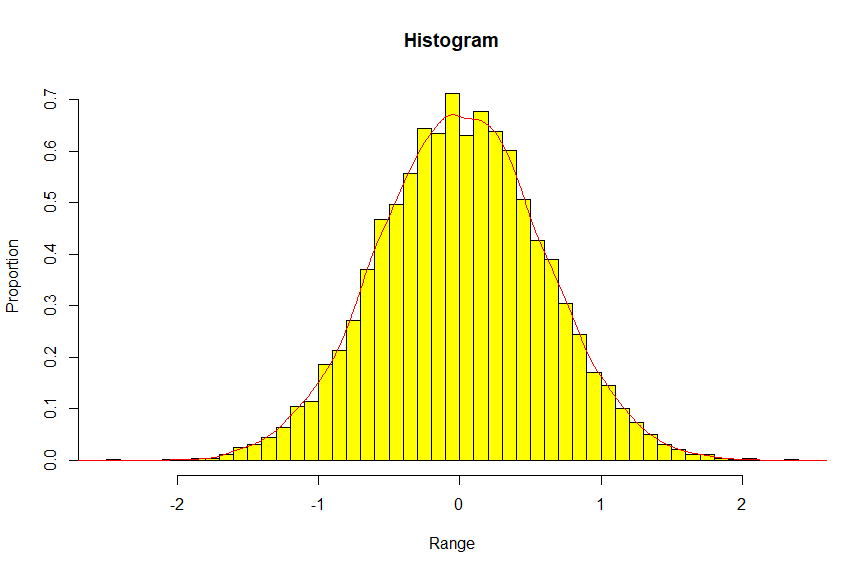}
	\caption{Histogram of the 10000 samples of $\sqrt{m}\sum_{i=1}^{n}w_{i,k}U_i$. Here we have $p=1$ with $n=10^4$, $m=2000$. The weight vector $W=(w_1,w_2,\cdots,w_n)^{\mathsf{T}}$ is simulated from $Dir((\frac{1999}{8000},\frac{1999}{8000},\cdots,\frac{1999}{8000}))$. The plot indicates the Gaussian nature of the samples.}
	\label{fig1}
\end{figure}
\end{example}

\begin{prop}\label{thm1}
	Let assumption~\ref{assm4} hold. In the regime $m/n\to 1$, we have
	\begin{align*}
	&\lim_{n\to \infty}\mathbb{E} \left|\sqrt{m}\Big(\sum_{i=1}^{n}w_i \nabla l(\theta,u_i) -\nabla g(\theta)\Big)-\sqrt{n}\Big(\sum_{i=1}^{n}\frac{1}{n}\nabla l(\theta,u_i) -\nabla g(\theta)\Big)\right|^2  =  0 
	\end{align*}
	and
	\begin{align*}
	% 		&\quad \quad \quad \quad \text{and}	\quad \quad \quad \quad \quad \quad \quad \quad \quad \quad \\
	&\lim_{n\to \infty}\mathbb{E}\left|\sqrt{n}(\sum_{i=1}^{n}w_i \nabla l(\theta,u_i) -\sum_{i=1}^{n}\frac{1}{n}\nabla l(\theta,u_i))\right|^2 = 0.
	\end{align*}
\end{prop}		
The proof is provided in the Appendix.
Using the above result, we instantly get the following CLT.
\begin{coro}
	In the regime $m/n\to 1$ as $n \to \infty$, we have \[\sqrt{m}\sum_{i=1}^{n} w_i \left(\nabla l(\theta,u_i)-\nabla g(\theta)\right) \xrightarrow{d} N(0,\sigma^2(\theta))\] as $n\to \infty$ where the weights $w_i$ are as in assumption~\ref{assm4}.
\end{coro}
The regime $m/n \to 1$ as $n \to \infty$ is much easier both in the intuitive and the technical senses. Proposition~\ref{thm1} considers this case. 
\subsection{Wasserstein Bounds for M-SGD}

Note that if we rewrite the iteration step in Algorithm~\ref{algo4} as $$x_{k+1}=x_k-\nabla g(x_k)+\frac{\gamma}{\sqrt{m}}\left(\sum_{i=1}^{n}\sqrt{m}w_{i,k}\left(\nabla l(x_k,u_{i,k})-\nabla g(x_k)\right)\right),$$ where the term $\sum_{i=1}^{n}\sqrt{m}w_{i,k}\left(\nabla l(x_k,u_{i,k})-\nabla g(x_k)\right)$ according to Theorem~\ref{cltthm2} is approximately normal given $x_k$. Hence, we might intuitively consider this algorithm to be equivalent to that given by $x_{k+1}\approx x_k-\nabla g(x_k)+\frac{\gamma}{\sqrt{m}} \sigma(x_k)Z_{k+1}$ where $Z_{k+1}$ is a standard Gaussian random variable in $p$ dimensions. This provides intuition for the following results which establishes that the dynamics of \eqref{sgd} and \eqref{msgd} as described by their continuous versions \eqref{intsgd} and \eqref{intmsgd} are close to the dynamics of a diffusion as described by \eqref{sgddiff}.
\subsubsection{The General Regime}
We establish a nonasymptotic bound between \eqref{intsgd} and \eqref{sgddiff} in the Wasserstein metric at any time point $t$ in the time horizon. We shall consider the convex and non-convex regimes separately as the treatment of the problem is somewhat different in each case.

\begin{theorem}\label{thmggg1}
	Suppose  assumptions~\ref{assm1}-\ref{assm5} hold. Recall $D_t$ and $X_t$ as the stochastic processes defined in equations~\eqref{intsgd} and~\eqref{sgddiff} respectively. Then for any $t\in (0,T]$ and any $m\ge 1$, we have
	\begin{align*}
	W^2_2(D_{t},X_t)\le C_{11}\gamma^2+C_{12}\gamma,
	\end{align*}
	where $C_{11},C_{12}$ are constants dependent only on $T,L,L_1,p$.
\end{theorem}
The proof of Theorem~\ref{thmggg1} is furnished in the Appendix where more information on the constants  $C_{11},C_{12}$ is also provided.
In Theorem~\ref{thmggg1} we establish that the Wassterstein distance between~\eqref{intmsgd} and~\eqref{intsgd} is in the order of the step-size.
There have been previous works attempting to address this problem~\cite{wu2020noisy}. These works derive bounds in settings which assume that the loss function is bounded. Here we do this in our set-up which assumes the loss function and the covariance function are Lipschitz in the parameter. Recall that here $W_2$ is the 2-Wasserstein distance which has been defined in the preliminaries section in~\eqref{wassdef}. Our main aim here is to show that the M-SGD algorithm is close to diffusion~\eqref{sgddiff} in the step size. The way we go about this is to construct a linear version of the algorithm~\eqref{intmsgd} and then show that this process is close to the diffusion~\eqref{sgddiff}. We have shown that $Y_{n,t}$ and $D_t$ as defined by equations \eqref{intmsgd} and \eqref{intsgd} are close in the Wasserstein distance in the Appendix. This brings us to one of our main results. 

\begin{theorem}\label{mainthmgg1}
	Suppose assumptions~\ref{assm1}-\ref{assm5} hold. Recall $Y_{n,t}$ and $X_t$ as the stochastic processes defined by~\eqref{intmsgd} and~\eqref{sgddiff}, respectively. Then for any  $t\in (0,T]$ and $m\ge 3^K$ we have
	\begin{align*}
	W_2^2(Y_{n,t},X_t)& \le C_{21}\gamma^2+C_{22}\gamma,
	\end{align*}
	where $C_{21},C_{22}$ are constants dependent only on $T,L,L_1,p$.
\end{theorem}
The proof of Theorem~\ref{mainthmgg1} is furnished in the Appendix where more information on the constants  $C_{21},C_{22}$ is also provided.

We observe that the M-SGD algorithm is close in distribution to~\eqref{sgddiff} at each point in order of the square root of the step size under strng conditions on $m$. We note that in Theorem~\ref{thmgg1} the dependence on $m$ is much weaker as in $m$ can take any value greater than 1. In Theorem~\ref{mainthmgg1} the minibatch size needs to be exponentially large in terms of the maximum iteration number. This is due to the fact that the distribution of the weights in Theorem~\ref{mainthmgg1} is unknown and hence needs a large sample and minibatch size to establish the same rate. For specific problems, we should be able to relax the condition on $m$. 
Next, we consider the case where the weights are non-negative. This case indeed improves the bound and makes the algorithm more applicable in practise.
\begin{theorem}\label{mainthmgg2}
	Suppose assumptions~\ref{assm1}-\ref{assm5} hold. Recall $Y_{n,t}$ and $X_t$ as the stochastic processes defined by~\eqref{intmsgd} and~\eqref{sgddiff}, respectively.  Then for any  $t\in (0,T]$ and $m\ge K^2$, with $w_{i,k}\ge 0$ for all $i,k$, we have
	\begin{align*}
	W_2^2(Y_{n,t},X_t)& \le C_{23}\gamma^2+C_{24}\gamma,
	\end{align*}
	where $C_{21},C_{22}$ are constants dependent only on $T,L,L_1,p$.
\end{theorem}
The proof is provided in the Appendix.

Theorem~\ref{mainthmgg2} shows that if we indeed have additional assumptions on the problem, we can derive weaker conditions on $m$ to obtain key bounds.
\subsubsection{The Convex Regime}
The next natural question is what the dynamic behavior of the M-SGD algorithm in the case where $g(\cdot)$ is strongly convex. The intuition is that this 
The next question that naturally arises is whether the M-SGD algorithm can be used as a tool for optimization. Since M-SGD is indeed an SGD algorithm, the answer to this question should be yes. However, the difference from previous SGD literature that we have in our setting is that the variance of the loss is not fixed but spatially varying and is Lipschitz. Also, we note that the objective function in question is strongly convex here and not the loss function. This has minor similarities with previous work~\cite{moulines2011non}. Invoking the assumption of strong convexity on $g(\cdot)$ we derive bounds for the convergence of the M-SGD algorithm in squared mean to the optimal point. Let $x^{*}=\argmin_{x}g(x)$.

\begin{prop}\label{optstrgcnvx}
	Taking Assumptons~\ref{assm1}-\ref{assm5} and~\ref{assm7} to be true, under the regime $0<\gamma <\min\left(1/L,1\right)$ and \[m>\frac{2pLL_1\gamma}{\lambda^2(2-L\gamma)},\] algorithms \eqref{sgd} and \eqref{msgd} exhibit
	\begin{align*}
	\mathbb{E}\left(g(v_{k+1})-g(x^*)\right) & \le \left[1-\lambda\gamma(2-L\gamma)+\frac{2pLL^2_1\gamma^2}{m\lambda}\right]^{k+1} \left(g(x_{0})-g(x^*)\right)\\
	&\quad +\frac{L\gamma}{m\left[\lambda(2-L\gamma)-\frac{2pLL^2_1\gamma}{m\lambda}\right]} \left|\left|\sigma(x^*)\right|\right|^2_F
	\end{align*}
	and
	\begin{align*}
	\mathbb{E}\left|v_{k+1}-x^*\right|^2 &\le \frac{2}{\lambda}\left[1-\lambda\gamma(2-L\gamma)+\frac{2pLL^2_1\gamma^2}{m\lambda}\right]^{k+1} \left(g(x_{0})-g(x^*)\right)\\
	&\quad +\frac{2}{\lambda}\left[\frac{L\gamma}{m\left(\lambda(2-L\gamma)-\frac{2pLL^2_1\gamma}{m\lambda}\right)}\right] \left|\left|\sigma(x^*)\right|\right|^2_F
	\end{align*}
	where $v_{k+1}$ is used to denote the $k+1$-th iterate of both \eqref{sgd} and \eqref{msgd}.
\end{prop}

The proof of Proposition~\ref{optstrgcnvx} is given in the Appendix. 

Note that the existence of an optima follows from strong convexity.
Also note that since $x^{*}$ is the optimum value, we have $g(x_0)-g(x^{*})>0$. This is not random as both $x_0$ and $x^{*}$ are deterministic points.

Proposition~\ref{optstrgcnvx} exhibits that under strong convexity of the main objective function (and not the loss function), one has geometric convergence of the online M-SGD algorithm and the SGD algorithm with scaled Gaussian error to the optimum. Note that the conditions on $\gamma$ and $m$ ensure that the rate of convergence is indeed less than $1$. Also note that our result assumes that the matrix $\sigma(\cdot)$ is Lipschitz which is vital to our proof. 

Now we are ready to state the main theorem of this section.
Define 
\[\rho=\left[1-\lambda\gamma(2-L\gamma)+\frac{2pLL^2_1\gamma^2}{m\lambda}\right].\]
\begin{theorem}\label{wass:conv:stng:convx}
    Take Assumptons~\ref{assm1}-\ref{assm5} and~\ref{assm7} as true, under the regime $0<\gamma <\min\left(1/L,1\right)$ and \[m>\frac{2pLL_1\gamma}{\lambda^2(2-L\gamma)}.\]
    Let $Y_{n,t}$ and $X_t$ denote the stochastic processes defined by~\eqref{intmsgd} and~\eqref{sgddiff}, respectively. Then for any  $t\in (0,T]$ we have
	\begin{align*}
	W_2^2(Y_{n,t},X_t)& \le \tilde{C}^{**}_1\rho^k+\tilde{C}^{**}_2\gamma^2+\tilde{C}^{**}_3\gamma
	\end{align*}
 for some constants $\tilde{C}^{**}_1$, $\tilde{C}^{**}_2$ and $\tilde{C}^{**}_3$ independent of $\gamma, \rho, m$. 
\end{theorem}
The proof is provided in the Appendix where more information on the constants can be found.

\begin{example}
  n this example we examine our algorithm in the context of the logistic regression problem. Consider $t \in \mathbb{N}$ and data given to us in the form $(y_i,x_i)_{i=1}^{t}\in \mathbb{R}^{p+1}$, where $y_i \in \{0,1\}$ and $x_i \in \mathbb{R}^p$. Our objective function is given as the negative log-likelihood plus an $\mathbb{L}_2$-regularization penalty. The objective function is 
\begin{align}\label{logisticobj}
g(\beta)=\frac{1}{t}\left[-\sum_{i=1}^{t}y_i x^{\mathsf{T}}_i\beta + \sum_{i=1}^{t}\log\left(1+e^{x^{\mathsf{T}}_i\beta}\right)\right]+\kappa \left|\beta\right|^2,
\end{align}
where $\kappa>0$ is some constant. 

We choose our training data to be the random samples of $(y_i,x_i)$ done with replacement. That is for each $u_i \in (u_1,u_2,\cdots,u_n)$, we have $u_i=(y_j,x_j)$ with probability $1/t$ for all $i,j$. The parameter for the problem is $\beta \in \mathbb{R}^p$. Note that this objective function as defined in \eqref{logisticobj}, is strongly convex with Lipschitz gradients. Indeed this is easy to see as 
\begin{align*}
\nabla^2 g(\beta)=\frac{1}{t}\left[\sum_{i=1}^{t}\frac{e^{x^{\mathsf{T}}_i\beta}}{\left(1+e^{x^{\mathsf{T}}_i\beta}\right)^2}x_ix^{\mathsf{T}}_i\right]+2\kappa I.
\end{align*}
It is immediate that the above matrix is positive definite with $||\nabla^2 g(\beta)||_2 \ge 2\kappa$. Also, as $x_i$ are fixed data points, $\left|\left|\nabla^2 g(\beta)\right|\right|_2 \le 1/t\, \lambda_{max}(XX^{\mathsf{T}})+2\kappa$, where $X=[x_1,x_2,\cdots,x_t]$ and $\lambda_{max}(XX^{\mathsf{T}})$ denotes the largest eigenvalue of $XX^{\mathsf{T}}$. Hence $\nabla g(\beta)$ is Lipschitz. 

Define $u_i=(v_i,u_{1,i},u_{2,i},\cdots,u_{p,i})^{\mathsf{T}}$ and $\tilde{u}_i=(u_{1,i},u_{2,i},\cdots,u_{p,i})^{\mathsf{T}}$. Also define the loss function as 
\begin{align}
l(\beta,u)=-v_i \tilde{u}^{\mathsf{T}}_i\beta +\log\left(1+e^{\tilde{u}^{\mathsf{T}}_i\beta}\right)+\kappa \left|\beta\right|^2.
\end{align}
Note that the loss function is also strongly convex and Lipschitz in $\beta$. It can also be easily seen that the loss function is unbiased for the objective function. We need to find a matrix $\sigma(\beta)$ such that $\sigma(\beta)\sigma(\beta)^{\mathsf{T}}=Var(\nabla l(\beta,u))$ and $\sigma(\beta)$ is Lipschitz in $\beta$ in the $||\cdot||_2$ norm. Now, 
\begin{align*}
\nabla l(\beta,u)=-v_i\tilde{u}_i+\frac{e^{\tilde{u}^{\mathsf{T}}_i\beta}}{1+e^{\tilde{u}^{\mathsf{T}}_i\beta}} u_i+2\kappa \beta.
\end{align*}
Define $z_i=(y_i,x_i)$. We have
\begin{align*}
Var(\nabla l(\beta,u))&=\frac{1}{t}\sum_{i=1}^{t}\left(\nabla l(\beta,z_i)-\nabla g(\beta)\right)\left(\nabla l(\beta,z_i)-\nabla g(\beta)\right)^{\mathsf{T}}.
\end{align*}
Define $$A=\frac{1}{\sqrt{t}}\left[\left(\nabla l(\beta,z_1)-\nabla g(\beta)\right),\left(\nabla l(\beta,z_2)-\nabla g(\beta)\right),\cdots, \left(\nabla l(\beta,z_t)-\nabla g(\beta)\right)\right].$$
Note that 
$$Var(\nabla l(\beta,u))=AA^{\mathsf{T}}.$$
Hence, for this problem, we may take $$\sigma(\beta)=\frac{1}{\sqrt{t}}\left[\left(\nabla l(\beta,z_1)-\nabla g(\beta)\right),\left(\nabla l(\beta,z_2)-\nabla g(\beta)\right),\cdots, \left(\nabla l(\beta,z_t)-\nabla g(\beta)\right)\right].$$
It can easily seen now that $\sigma(\beta)$ is Lipschitz in the Frobenius norm. We have
\begin{align*}
\left|\left|\sigma(\beta_1)-\sigma(\beta_2)\right|\right|_F \le \frac{1}{\sqrt{t}}\sum_{i=1}^{t}\left|\nabla l(\beta_1,z_i)-\nabla l(\beta_2,z_i)\right|+\sqrt{t}\left|\nabla g(\beta_1)-\nabla g(\beta_2)\right|.
\end{align*}
As $\nabla l$ and $\nabla g$ are both Lipschitz in $\beta$, we have $\sigma(\beta)$ as a Lipschitz function in $\beta$.

Note that the above argument for $\sigma(\beta)$ being Lipscitz can be applied to large class of problems with such variance covariance matrix. The only two conditions necessary to establish this is that both $\nabla l(\cdot,z)$ and $\nabla g(\cdot)$ are Lipschitz. Also an implicit assumption in this case is that the data is fixed. 

We provide simulation example to exhibit convergence for the algorithm. Consider $p=6$ with the number of data  points as $t=10^4$. The number of samples we choose randomly with replacement is $n=10^3$ and the minibatch size is $m=10$. We consider $5$ values of $\kappa$ as $(.2,.1,.05,.01,.001)$. The data is generated as $y_i \sim Ber(1/2)$ iid and $x_i$ are random standard Gaussian. The weights at each step of the iteration are generated as per $W\sim N(\mu,\Sigma)$ where $\mu$ and $\Sigma$ are provided in assumption~\ref{assm4}. Note that the true $\beta=0$. After each iteration is complete we replicate it and take the norm of all the replicated $\hat \beta$ and take their average. This gives us an approximation of $E|\beta|^2$. We plot this and show that it converges to $0$ at different rates which depend on $\kappa$.
\begin{figure}[h]
\begin{minipage}[h]{0.47\linewidth}
\begin{center}
\includegraphics[width=1\linewidth]{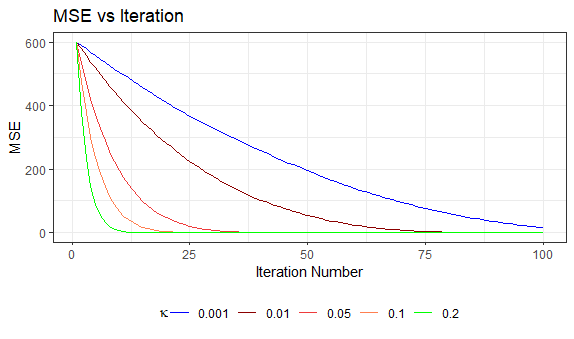} 
\caption{MSE vs Iteration with $\gamma=.5$}
\label{gamma_point_5}
\end{center} 
\end{minipage}
\hfill
\vspace{0.2 cm}
\begin{minipage}[h]{0.47\linewidth}
\begin{center}
\includegraphics[width=1\linewidth]{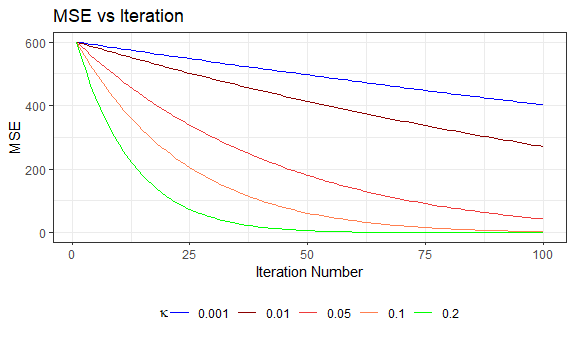} 
\caption{MSE vs Iteration with $\gamma=.1$}
\label{gamma_point_1}
\end{center}
\end{minipage}
\end{figure}   
\end{example}
\section{Discussion}
%\hspace{.5 cm} 
In our findings we have exhibited that the M-SGD error is approximately Gaussian irrespective of the distribution of the weights used in the problem as long as the number of samples and the number of minibatch is large. This helps practitioners comprehend that the dynamics of the M-SGD algorithm is very similar to that of the SGD algorithm with a scaled Gaussian error.  Our results exhibit that the M-SGD algorithm is close in distribution to a particular diffusion, the dynamics of which is somewhat known. Diffusions generally exhibit the interesting phenomenon of escaping low potential regions~\cite{hu2018diffusion}. Our work has been the first step in establishing that the M-SGD algorithm also escapes low potential regions.

Note that few questions naturally arise from our work. The first natural question is whether the Gaussian nature of the M-SGD error tight. We conjecture that this is indeed the case. The second question is to obtain sharper bounds for more specific class of problems. We believe that this is also possible. The third question is whether the M-SGD algorithm can be tweaked so that the dimension dependence and the dependence on the time horizon $T$ improves. We shall try to address these questions in future research work.

\nocite{*} 
\bibliographystyle{plain}
\bibliography{mcref}

%%Appendix
\begin{appendices}\label{Appendix}

\end{appendices}
\section{Appendix} 
\subsection*{Proofs of Proposition~\ref{thm1} and Theorem~\ref{cltthm2}}
We present the following lemma which shall be used throughout our work.
\begin{lemma}\label{lemmamain}
	Under assumption~\ref{assm2}, we have 
	\begin{align*}
	\mathbb{E} \left(\nabla l(\theta,u)\right)=\nabla g(\theta).
	\end{align*}
\end{lemma}
\begin{proof}
	We prove this for 1 dimension as that suffices for the general case as  expectation distributes over all the components of the vector. Here $\theta$ is a fixed point at which we differentiate. Note that 
	\begin{align*}
	\frac{l(\theta_1,u)-l(\theta,u)}{\theta_1-\theta}=\nabla l(\xi,u)
	\end{align*}
	for some $\xi$ using the mean value theorem. Note that, since differentiation is a local property, we can force $\theta_1 \in B(\theta,1)$, where $B(\theta,1)$ denotes the ball centred at $\theta$ with radius $1$. This also forces $\xi \in B(\theta,1)$. Also note,
	\begin{align*}
	\left|\nabla l(\xi,u)\right|&\le \left|\nabla l(\theta,u)\right|+h_1(u)\left|\xi-\theta\right|\\
	& \le \left|\nabla l(\theta,u)\right|+h_1(u).
	\end{align*}
	The last line follows as $\xi \in B(\theta,1)$.	This implies 
	\begin{align*}
	\frac{l(\theta_1,u)-l(\theta,u)}{\theta_1-\theta} & \le \left|\frac{l(\theta_1,u)-l(\theta,u)}{\theta_1-\theta}\right|\\
	&\le \left|\nabla l(\theta,u)\right|+h_1(u).
	\end{align*}
	The last term is independent of $\theta_1$ and is integrable. Hence we can use DCT and we are done.
\end{proof}

\begin{proof}[Proof of Proposition~\ref{thm1}]
	We begin by noting the fact that
	\begin{align*}
	\mathbb{E}\left|\sqrt{m}\left(\sum_{i=1}^{n}w_i \nabla l(\theta,u_i) -\nabla g(\theta)\right)-\sqrt{n}\left(\sum_{i=1}^{n}\frac{1}{n}\nabla l(\theta,u_i) -\nabla g(\theta)\right)\right|^2 & \\
	\quad=\mathbb{E}\left|\sum_{i=1}^{n}\left(\sqrt{m}w_i-\frac{1}{\sqrt{n}}\right)\nabla l(\theta,u_i)+\left(\sqrt{n}-\sqrt{m}\right)\nabla g(\theta)\right|^2.
	%& \\	
	\end{align*}
	Now, the last term is equal to
	\begin{align*}
	&\mathbb{E}\left\langle\sum_{i=1}^{n}\left(\sqrt{m}w_i-\frac{1}{\sqrt{n}}\right)\nabla l(\theta,u_i),\sum_{i=1}^{n}\left(\sqrt{m}w_i-\frac{1}{\sqrt{n}}\right)\nabla l(\theta,u_i)\right\rangle\\
	&+ 2 \ \mathbb{E}\left\langle\sum_{i=1}^{n}\left(\sqrt{m}w_i-\frac{1}{\sqrt{n}}\right)\nabla l(\theta,u_i),\left(\sqrt{n}-\sqrt{m}\right)\nabla g(\theta)\right\rangle
	+  \left(\sqrt{n}-\sqrt{m}\right)^2\left|\nabla g(\theta)\right|^2.
	\end{align*}
	We condition on $w=(w_1,w_2,\dots,w_n)$ and get the above expression equal to
	\begin{align*}
	& \mathbb{E}\left[\sum_{1\le i,j \le n}\left(\sqrt{m}w_i-\frac{1}{\sqrt{n}}\right)\left(\sqrt{m}w_j-\frac{1}{\sqrt{n}}\right)\mathbb{E}_w\left\langle\nabla l(\theta,u_i),\nabla l(\theta,u_j)\right\rangle\right]  \\
	&+ 2 \ \mathbb{E}\left[\sum_{i=1}^{n}\left(\sqrt{m}w_i-\frac{1}{\sqrt{n}}\right)\left(\sqrt{n}-\sqrt{m}\right)\left|\nabla g(\theta)\right|^2)\right]\\
	&+\left(\sqrt{n}-\sqrt{m}\right)^2 \ \left|\nabla g(\theta)\right|^2.
	\end{align*}	
	Here $\mathbb{E}_w$ denotes the conditional expectation with respect to the weights. Using the fact $\mathbb{E}(w_i)=1/n$ and some minor manipulation, the second term is $-2\left(\sqrt{n}-\sqrt{m}\right)^2\left|\nabla g(\theta)\right|^2$.
	Hence, we have
	\begin{align*}
	& \mathbb{E}\left[\sum_{1\le i,j \le n}\left(\sqrt{m}w_i-\frac{1}{\sqrt{n}}\right)\left(\sqrt{m}w_j-\frac{1}{\sqrt{n}}\right)\mathbb{E}_w\left\langle\nabla l(\theta,u_i),\nabla l(\theta,u_j)\right\rangle\right]\\
	&\quad -\left(\sqrt{n}-\sqrt{m}\right)^2 \ \left|\nabla g(\theta)\right|^2\\
	&= \mathbb{E}\Bigg[\sum_{i=1}^{n}\left(\sqrt{m}w_i-\frac{1}{\sqrt{n}}\right)^2 \mathbb{E}_w \left|\nabla l(\theta,u_i)\right|^2\\
	&\quad +\sum_{1\le i,j \le n, i\ne j} \left(\sqrt{m}w_i-\frac{1}{\sqrt{n}}\right)\left(\sqrt{m}w_j-\frac{1}{\sqrt{n}}\right)\mathbb{E}_w\left(\nabla l(\theta,u_i)^{\mathsf{T}}\nabla l(\theta,u_j)\right)\Bigg]\\
	& \quad -\left(\sqrt{n}-\sqrt{m}\right)^2 \ \left|\nabla g(\theta)\right|^2.
	\end{align*}
	The final term equals
	\begin{align*}
	&\mathbb{E}\Bigg[\sum_{i=1}^{n}\left(\sqrt{m}w_i-\frac{1}{\sqrt{n}}\right)^2 \left(\Tr \sigma^2(\theta)+\left|\nabla g(\theta)\right|^2\right)\\
	&\quad +\sum_{1\le i,j \le n}\left(\sqrt{m}w_i-\frac{1}{\sqrt{n}}\right)\left(\sqrt{m}w_j-\frac{1}{\sqrt{n}}\right) \left|\nabla g(\theta)\right|^2\\
	& \quad -\left(\sqrt{n}-\sqrt{m}\right)^2 \ \left|\nabla g(\theta)\right|^2\Bigg].
	\end{align*}
	Using the covariance structure of the weights, the last expression is reduced to 
	\begin{align*}
	& 2\left[\Tr \sigma^2(\theta)+\left|\nabla g(\theta)\right|^2\right]\left(1-\sqrt{\frac{m}{n}}\right)+\left|\nabla g(\theta)\right|^2\left[(\sqrt{m}-\sqrt{n})^2-2+2\sqrt{\frac{m}{n}}\right]\\
	&\quad -\left(\sqrt{n}-\sqrt{m}\right)^2\left|\nabla g(\theta)\right|^2\\
	&\quad =2\left(1-\sqrt{\frac{m}{n}}\right)\Tr \sigma^2(\theta).
	\end{align*}
	Using this, in the regime $m/n\to 1$ as $n\to \infty$, we get the first conclusion for Proposition~\ref{thm1}.
	
	The second conclusion to Proposition~\ref{thm1} can be derived similarly.
\end{proof}

\begin{lemma}\label{firstlemma}
	Let assumption~\ref{assm2} hold. Also let assumptions \ref{assm4} and~\ref{assm6case1} hold. Then, we have  $$m\sum_{i=1}^{n} w_i^2 \xrightarrow[L^1]{a.s.} 1 ,$$ when $m/n\to \gamma^*$ and $0\le \gamma^*\le 1$.	
\end{lemma}
\begin{proof} 
	Noticing the fact that $w=\Sigma^{1/2}X+1/n\,l$, we have 
	$$m\sum_{i=1}^{n} w_i^2=m\left(X^{\mathsf{T}}\Sigma X \ +\frac{1}{n}\right).$$
	We also note that \[\Sigma=\frac{n-m}{mn(n-1)}\left(I-\frac{1}{n}ll^{\mathsf{T}}\right).\]
	This leads to that
	\begin{align*}
	m\sum_{i=1}^{n} w_i^2  =& \frac{n-m}{n(n-1)}X^{\mathsf{T}}\left(I-\frac{1}{n}ll^{\mathsf{T}}\right)X+\frac{m}{n}\\
	=& \frac{n-m}{n-1}\left(\frac{1}{n}X^{\mathsf{T}}X-\bar{X}^2\right)+\frac{m}{n}.
	\end{align*}
	\beaa
	\eeaa
	Now, the above expression converges to $1$ both almost surely and in $\mathbb{L}_1$-norm, whatever $\gamma^*$ is. If $\gamma^*=0$, the second term converges to $0$ and the first term converges to $1$ almost surely using Law of Large Numbers(also $\mathbb{L}_1$ as $X_1$ is sub-Gaussian). If $\gamma^*=1$, the second term converges to $1$ and the first term converges to 0. If $0<\gamma^*<1$, we have the second term converging to $\gamma^*$ and the first converging to $1-\gamma^*$. 
 Hence we conclude the proof.
\end{proof}
Note that the matrix $\Sigma=U\Lambda U^{\mathsf{T}}$ where $$\Lambda=\frac{n-m}{mn(n-1)}\begin{bmatrix}
I_{n-1} & 0_{n-1}\\
0^{\mathsf{T}}_{n-1} & 0
\end{bmatrix},$$
and we can choose $U=[x_1,x_2,\dots,x_n]$, such that 
\begin{align}\label{xidef}
x_i=\sqrt{\frac{n-i}{n-i+1}}\cdot \left(0,0,\dots,1,-\frac{1}{n-i},-\frac{1}{n-i},\dots,-\frac{1}{n-i}\right)^{\mathsf{T}},
\end{align} 
i.e., a scalar times first $i-1$ entries $0$, $1$ in the $i^{th}$ entry and the rest $1/n-i$, for $1\le i\le n-1$ and $x_n=l/\sqrt{n}$. Also note that \[\Sigma^{1/2}=\sqrt{\frac{n-m}{mn(n-1)}}\cdot \left(I_n-\frac{1}{n}ll^{\mathsf{T}}\right).\]
\begin{lemma}\label{secondlemma}
	Let assumption~\ref{assm2} hold. Also let assumptions \ref{assm4} and~\ref{assm6case1} hold. Then, we have $$m^{3/2}\sum_{i=1}^{n}|w_i|^3 \xrightarrow[]{L^1}  0$$ as $n\to \infty$ with $\frac{m}{n} \to \gamma^*$ and $0\le \gamma^* \le 1$.
\end{lemma}
\begin{proof}
	 We begin our proof by observing that $w_i$ all have the same distribution. This is easy to see, using the fact the $X_i$ are iid and \[w_i=\sum_{j=1}^{n-1}\sqrt{\frac{n-m}{mn(n-1)}}\left(x_j^{\mathsf{T}}X\right)x_{j,i}+\frac{1}{n}\] (this follows from $W=\Sigma^{1/2}X+1/n\,l$), where $x_j$ are as defined in \eqref{xidef}. Also note that we can take $X_i$ to have $0$ mean as $\Sigma^{1/2}X=\Sigma^{1/2}\left(X-\mu \cdot l\right)+\Sigma^{1/2}\mu \cdot l=\Sigma^{1/2}\left(X-\mu \cdot l\right)$. The last step follows as $\mu$ is a scalar and $\Sigma^{1/2}l=0$. With this we have
	\begin{align*}
	m^{3/2}\sum_{i=1}^{n}\left|w_i\right|^3 &= m^{3/2}\sum_{i=1}^{n}\left|\sum_{j=1}^{n-1}\sqrt{\frac{n-m}{mn(n-1)}}\left(x_j^{\mathsf{T}}X\right)x_{j,i}+\frac{1}{n}\right|^3\\
	&\le 4 \left(\frac{n-m}{n(n-1)}\right)^{3/2}\sum_{i=1}^{n}\left|\sum_{j=1}^{n-1}\left(x_j^{\mathsf{T}}X\right)x_{j,i}\right|^3+ 4\frac{m^{3/2}}{n^2}.
	%\\
	\end{align*}
	It is easy to see that the second term, as for $0\le \gamma^* \le 1$, converges to 0.
	
	Define $$T_i=\sum_{j=1}^{n-1}\left(x_j^{\mathsf{T}}X\right)x_{j,i}.$$ We need to check that $$4 \left(\frac{n-m}{n(n-1)}\right)^{3/2}\sum_{i=1}^{n}\mathbb{E}|T_i|^3 \to 0 \ \text{as} \ n\to \infty.$$ We show $\frac{1}{n^{3/2}}\sum_{i=1}^{n}\mathbb{E}|T_i|^3 \to 0 \ \text{as} \ n\to \infty$. First, 
	$$
	\frac{1}{n^{3/2}}\sum_{i=1}^{n}\mathbb{E}|T_i|^3 = \frac{1}{\sqrt n} \,\mathbb{E}|T_1|^3 ,
	$$
	which follows from the fact that $T_i$ are just centered and scaled $w_i$ and hence have the same distribution. Also,
	\begin{align*}
	T_i=\sum_{j=1}^{n-1}\left(x_j^{\mathsf{T}}X\right)x_{j,i}
	= \sum_{j=1}^{n-1}\left(\sum_{k=1}^{n}x_{j,k}X_k\right)x_{j,i}
	= \sum_{k=1}^{n}\left(\sum_{j=1}^{n-1}x_{j,i}x_{j,k}\right)X_k.
	\end{align*}
	By using this and the Hoeffding inequality we get,
	$$\mathbb{P}\left(|T_1|>t\right)\le 2\ \exp\left\{-\frac{ct^2}{K^2\sum_{k=1}^{n}(\sum_{j=1}^{n-1}x_{j,1}x_{j,k})^2}\right\},$$
	where $K$ is a positive constant which depends on the distribution of $X_1$ and $c>0$ is another such positive constant.
	
	Note that our choice of $x_i$ ensure that $x_{j,i}=0$ when $j>i$. Thus, in this case, $x_{1,1}$ is the only non-zero value. Also note that by our construction, $x_{1,1}=\sqrt{ (n-1)/n}$. Thus $$\mathlarger{\mathlarger{\sum}}_{k=1}^{n}\left(\sum_{j=1}^{n-1}x_{j,1}x_{j,k}\right)^2=\left(1-\frac{1}{n}\right)\sum_{k=1}^{n}x_{j,k}^2=1-\frac{1}{n}.$$
	This implies that $$\mathbb{P}\left(|T_1|>t\right)\le 2\exp\left\{-\frac{ct^2}{K}\right\}$$ for some constants $c,K$. Hence, there exists some constant $C >0$ such that
	$$\mathbb{E}|T_1|^3 \le C .$$
	Hence $\frac{1}{\sqrt n} \,\mathbb{E}|T_1|^3 \to 0 \ \text{as} \ n\to \infty$.
	Hence we conclude the proof for assumption~\ref{assm6case1}.
\end{proof}
\begin{lemma}\label{thirdlemma}
	Let assumption~\ref{assm2} hold. Also let assumptions \ref{assm4} and~\ref{assm6case1}   or \ref{assm4} and~\ref{assm6case2} hold. In the regime $\frac{m}{n}\to \gamma^* \ \text{with} \ 0\le \gamma^* \le 1$, we have $$\mathbb{E}\left[\frac{\sum_{i=1}^{n}|w_i|^3}{\left(\sum_{i=1}^{n}w^2_i\right)^{3/2}}\right]\to 0 \quad \text{as}\  n\to \infty .$$
\end{lemma}
\begin{proof}
	We know the inequality $$\sum_{i=1}^{n}a^p \le \left(\sum_{i=1}^{n}a_i\right)^p ,$$ where $a_i \ge 0$ and $p>1$. Using this, we know that 
	$$\frac{\sum_{i=1}^{n}|w_i|^3}{\left(\sum_{i=1}^{n}w^2_i\right)^{3/2}} \le 1 .$$ We also know from Lemma~\ref{firstlemma} and Lemma~\ref{secondlemma}, $m^{3/2}\sum_{i=1}^{n}|w_i|^3 \xrightarrow{P} 0 \ \text{as} \ n \to \infty$ and $m\sum_{i=1}^{n} w_i^2\xrightarrow{a.s.} 1 \ \text{as} \ n \to \infty$. Hence,
	$$\frac{\sum_{i=1}^{n}|w_i|^3}{(\sum_{i=1}^{n}w^2_i)^{3/2}} \xrightarrow{P} 0 \quad \text{as} \quad n \to \infty .$$
	Using DCT, we are done. Note that this proof works for both assumption~\ref{assm6case1} and assumption~\ref{assm6case2}.
\end{proof}

\begin{proof}[Proof of Theorem~\ref{cltthm2}]
	Let us consider $p=1$, where $p$ is the dimension. 
	Consider,
	$$\left|\mathbb{P}\left(\sqrt{m}\sum_{i=1}^{n}w_i(\nabla l(\theta,u_i)-\nabla g(\theta)) \le x\right)-\Phi_{\sigma}(x)\right|,$$
	where $\Phi_{\sigma}(x)$ is the cdf of $N(0,\sigma^2(\theta))$. 
	Define $\left(\nabla l(\theta,u_i)-\nabla g(\theta)\right)=X_i$. Thus $X_i$ are iid with mean $0$. 
	For this particular case, we take without loss of generality, $\sigma^2(\theta)=\sigma^2=1$. 
	Therefore the problem reduces to proving  
	$$\left|\mathbb{P}\left(\sqrt{m}\sum_{i=1}^{n}w_i \ X_i \le x\right)-\Phi(x)\right|$$
	goes to zero where $\Phi$ is the cdf of standard normal.
	Now,
	\begin{align*}
	&\Big|\mathbb{E} \ \mathbb{P}\Big(\sqrt{m}\sum_{i=1}^{n}w_i \ X_i \le x \mid w\Big)-\mathbb{E} \left(\Phi(x)\right)\Big|\\
	& \le  \mathbb{E} \ \Big|\mathbb{P}\Big(\sqrt{m}\sum_{i=1}^{n}w_i \ X_i \le x \mid w\Big)- \Phi(x)\Big|\\
	& \le \mathbb{E} \left[\ \frac{\sum_{i=1}^{n}\mathbb{E} \ [|w_i \ X_i|^3 \mid w]}{(\sum_{i=1}^{n}w_i^2)^{3/2}}\right]+\mathbb{E}\left|\Phi\left(\frac{x}{m\sum_{i=1}^{n}w^2_i}\right)-\Phi(x)\right|\\
	& =  \mathbb{E}|X_1|^3 \  \mathbb{E} \left[\frac{\sum_{i=1}^{n}|w_i| ^3}{(\sum_{i=1}^{n}w_i^2)^{3/2}}\right] +\mathbb{E}\left|\Phi\left(\frac{x}{m\sum_{i=1}^{n}w^2_i}\right)-\Phi(x)\right|.\\
	\end{align*}
Now, it has been proved in Lemma~\ref{thirdlemma} that the first term goes to $0$. Using $m\sum_{i=1}^{n}w^2_i \xrightarrow{a.s.} 1 \quad \text{as} \quad n \to \infty$ and the fact that $||\Phi||_{\infty}<1$, we get that the second term converges to zero as well using DCT. We can extend to any dimension using Cramer-Wold device. Hence the proof is completed.
\end{proof}
\subsection*{Proofs for positive weights with Dirichlet example}
\begin{theorem}\cite[Corollary 3.1]{arenal1996zero}\label{prev:result:main}
    Let $X_n; \ n=1,2,\cdots$ be a sequence of random variables which are iid with $Var(X_n)=\sigma^2$. If $w_n$ is a sequence of weight vectors satisfying Assumption~\ref{assm6case2}, then 
    \[\sqrt{m}\left(\sum_{i=1}^{n}w_{n,i} X_i -\frac{1}{n}\sum_{i=1}^{n}X_i\right)\overset{d}{\rightarrow} N(0,\sigma^2).\]
\end{theorem}

\begin{proof}[Proof of Theorem~\ref{positive:wts:thm}]
    The proof follows easily from Theorem~\ref{prev:result:main} by noting that 
    \begin{align*}
        \sqrt{m}\left(\sum_{i=1}^{n}w_{n,i} X_i -\frac{1}{n}\sum_{i=1}^{n}X_i\right)=\sqrt{m}\sum_{i=1}^{n}w_{n,i} X_i-\sqrt{\frac{m}{n}}\frac{1}{\sqrt{n}}\sum_{i=1}^{n}X_i.
    \end{align*}
    We know that since $X_i$ are iid mean zero with second moment and $m/n \to 0$, $\sqrt{\frac{m}{n}}\frac{1}{\sqrt{n}}\sum_{i=1}^{n}X_i \overset{P}{\rightarrow} 0$ as $n\to \infty$. Therefore the result follows using Slutsky's Theorem.
\end{proof}
\begin{prop}\label{clt:assm:prop}
    	If the random variables $w_i$ are exchangeable and $w_i \ge 0$ with
	$$\mathbb{E}(w^3_i) \le \frac{o\left(m^{-3/2}\right)}{n}, \ \mathbb{E}\left(w^4_i\right)\le \frac{o(m^{-2})}{n} \ \text{and} \ \mathbb{E}\left(w^2_i w^2_j\right)=\left(\frac{1}{m\,n}\right)^2+o\left(\left(\frac{1}{m\,n}\right)^2\right),$$
	then Assumption~\ref{assm6case2} holds if $m/n \to 0$ as $n\to \infty$.
\end{prop}
\begin{proof}[Proof of Proposition~\ref{clt:assm:prop}]
   Note that we only need to establish the last two points in Assumption~\ref{assm6case2}. For the first point we have for any $\epsilon$
   \begin{align*}
       \mathbb{P}\left(\sqrt{m}\max_{1\le j\le n}\left|w_i -\frac{1}{n}\right|\ge \epsilon\right)&\le n\, \mathbb{P}\left(\sqrt{m}\left|w_1-\frac{1}{n}\right|\ge \epsilon\right)\\
       &\le \frac{1}{\epsilon^3}\,n\, m^{3/2}\, \mathbb{E}\left|w_i -\frac{1}{n}\right|^3\\
       &\le \frac{4\,n\, m^{3/2}}{\epsilon^3} \left(\mathbb{E}\left|w_i\right|^3 +\frac{1}{n^3}\right).
   \end{align*}
   By our hypothesis the right hand side goes to $0$ as $n\to \infty$.
   
   For the second condition, define
  $Y_i=m \ w_i$. Using this and the fact that $w_i$ are exchangeable, we have
	\begin{align*}
	\mathbb{E}\left(\frac{1}{m}\sum_{i=1}^{n}Y^2_i\right)=\frac{n}{m} \mathbb{E}(Y^2_1)=1.
	\end{align*}
	And 
	\begin{align*}
	Var\left(\frac{1}{m}\sum_{i=1}^{n}Y^2\right)&=\mathbb{E}\left(\frac{1}{m}\sum_{i=1}^{n}Y^2_i\right)^2-\left[\mathbb{E}\left(\frac{1}{m}\sum_{i=1}^{n}Y^2_i\right)\right]^2\\
	& = \mathbb{E}\left(\frac{1}{m}\sum_{i=1}^{n}Y^2_i\right)^2-1 \\
	&=\frac{1}{m^2} \,\mathbb{E}\left(\sum_{i=1}^{n}Y^4_i\right)+\frac{1}{m^2} \, \mathbb{E}\left(2\sum_{1\le i<j \le n}Y^2_i Y^2_j\right)-1\\
	&=\frac{n}{m^2} \,\mathbb{E}(Y^4_1)+\frac{n(n-1)}{m^2} \,\mathbb{E}(Y^2_1 Y^2_2)-1\\
	&\le \frac{o(m^2)}{m^2}+\frac{n(n-1)}{m^2}\left(o\left(\left(\frac{m}{n}\right)^2\right)+\left(\frac{m}{n}\right)^2\right)-1.
	\end{align*}
Therefore $m\sum_{i=1}^{n}w^2_i \overset{P}{\rightarrow} 1$ as $n\to \infty$. This implies
\begin{align*}
    m\sum_{i=1}^{n} \left(w_i-\frac{1}{n}\right)^2&=m\, \sum_{i=1}^{n}w^2_i-\frac{2\,m}{n}\sum_{i=1}^{n}w_i+\frac{m}{n}\\
    &=m\, \sum_{i=1}^{n}w^2_i-\frac{m}{n}.
\end{align*}
Using the fact that $m/n \to 0$ as $n\to \infty$ and and  $m\sum_{i=1}^{n}w^2_i \overset{P}{\rightarrow} 1$, we are done.  
\end{proof}
\begin{proof}[Proof of Proposition~\ref{Dirwtlemma}]
We shall make use of Proposition~\ref{clt:assm:prop} and the following lemma which we state without proof as the proof of the lemma is simple.
\begin{lemma}\label{dirlemmaeasy}
	If $X\sim Dir(\alpha)$ where $\alpha=(\alpha_1,\alpha_2,\cdots,\alpha_n)$, then $$\mathbb{E}\left(\prod_{i=1}^{n}X^{\beta_i}_i\right)=\frac{\Gamma\left(\sum_{i=1}^{n}\alpha_i\right)}{\Gamma\left(\sum_{i=1}^{n}(\alpha_i+\beta_i)\right)}\times \prod_{i=1}^{n}\frac{\Gamma\left(\alpha_i+\beta_i\right)}{\Gamma(\alpha_i)}.$$
\end{lemma}  
	We start with the third moment 
	\begin{align*}
	\mathbb{E}(Y^3_i)&=m^3 \,\mathbb{E}(w^3_i)=m^3 \,\mathbb{E}(w^3_1)\\
	&=m^3\frac{\Gamma\left(\sum_{i=1}^{n}\frac{m-1}{n-m}\right)}{\Gamma\left(\sum_{i=1}^{n}\frac{m-1}{n-m}+3\right)}\cdot \frac{\Gamma\left(\frac{m-1}{n-m}+3\right)}{\Gamma\left(\frac{m-1}{n-m}\right)}\\
	& =\frac{m^3}{n}\frac{\left(\frac{m-1}{n-m}+2\right)\left(\frac{m-1}{n-m}+1\right)}{\left(\frac{n(m-1)}{n-m}+2\right)\left(\frac{n(m-1)}{n-m}+1\right)}.
	\end{align*}
	This implies,
	\begin{align*}
	\mathbb{E}(Y^3_i)
	& \le \frac{m^3}{n} \left(\frac{1}{n}+2\frac{n-m}{n(m-1)}\right)\left(\frac{1}{n}+\frac{n-m}{n(m-1)}\right)\\
	& \le \frac{m}{n} \left(\frac{m}{n}+2\frac{m(n-m)}{n(m-1)}\right)\left(\frac{m}{n}+\frac{m(n-m)}{n(m-1)}\right)\\
	&=O\left(\frac{m}{n}\right),
	\end{align*}
	where the second step follows from the Lemma~\ref{dirlemmaeasy}. Hence we establish $\mathbb{E}(Y^3_i) \le o(m^{3/2})/n$. 
	We can similarly argue for $\mathbb{E}(Y^4_i)$. In fact,
	\begin{align*}
	\mathbb{E}(Y^4_i)&=m^4 \,\mathbb{E}(w^4_i)=m^4 \,\mathbb{E}(w^4_1)\\
	&=m^4\frac{\Gamma\left(\sum_{i=1}^{n}\frac{m-1}{n-m}\right)}{\Gamma\left(\sum_{i=1}^{n}\frac{m-1}{n-m}+4\right)}\cdot \frac{\Gamma\left(\frac{m-1}{n-m}+4\right)}{\Gamma\left(\frac{m-1}{n-m}\right)}.
	\end{align*}
	This implies that,
	\begin{align*}
	\mathbb{E}(Y^4_i)
	& =\frac{m^4}{n}\frac{\left(\frac{m-1}{n-m}+3\right)\left(\frac{m-1}{n-m}+2\right)\left(\frac{m-1}{n-m}+1\right)}{\left(\frac{n(m-1)}{n-m}+3\right)\left(\frac{n(m-1)}{n-m}+2\right)\left(\frac{n(m-1)}{n-m}+1\right)}\\
	& \le \frac{m^4}{n} \left(\frac{1}{n}+3\frac{n-m}{n(m-1)}\right) \left(\frac{1}{n}+2\frac{n-m}{n(m-1)}\right)\left(\frac{1}{n}+\frac{n-m}{n(m-1)}\right)\\
	& =\frac{m}{n} \left(\frac{m}{n}+3\frac{m(n-m)}{n(m-1)}\right)\left(\frac{m}{n}+2\frac{m(n-m)}{n(m-1)}\right)\left(\frac{m}{n}+\frac{m(n-m)}{n(m-1)}\right)\\
	&=O\left(\frac{m}{n}\right).
	\end{align*}
	Hence we establish the required condition for the fourth power as $\mathbb{E}(Y^4_i)=o(m^2)/n$. The last step is to show that $\mathbb{E}(Y^2_iY^2_j)$ satisfies the assumption. Now
	\begin{align*}
	\mathbb{E}(Y^2_iY^2_j)&=m^4 \,\mathbb{E}(w^2_iw^2_j)=m^4 \,\mathbb{E}(w^2_1w^2_2)\\
	&=m^4\frac{\Gamma\left(\sum_{i=1}^{n}\frac{m-1}{n-m}\right)}{\Gamma\left(\sum_{i=1}^{n}\frac{m-1}{n-m}+4\right)}\cdot \left[\frac{\Gamma\left(\frac{m-1}{n-m}+2\right)}{\Gamma\left(\frac{m-1}{n-m}\right)}\right]^2.
	\end{align*}
	This leads to
	\begin{align*}
	\mathbb{E}(Y^2_iY^2_j)
	& =m^4\frac{\left(\frac{m-1}{n-m}+1\right)^2\left(\frac{m-1}{n-m}\right)^2}{\left(\frac{n(m-1)}{n-m}+3\right)\left(\frac{n(m-1)}{n-m}+2\right)\left(\frac{n(m-1)}{n-m}+1\right)\frac{n(m-1)}{n-m}}\\
	& \le \frac{m^4}{n^2} \left(\frac{1}{n}+\frac{n-m}{n(m-1)}\right)^2\\
	& =\frac{m^2}{n^2}\left[\left(\frac{m}{n}+\frac{m(n-m)}{n(m-1)}\right)^2-1\right]+\frac{m^2}{n^2} \\
	&=o\left(\frac{m^2}{n^2}\right)+\frac{m^2}{n^2}.
	\end{align*}\\
	Hence assumption~\ref{assm6case2} is satisfied. This implies that the CLT holds for Dirichlet weights with the parameter vector $(\frac{m-1}{n-m},\frac{m-1}{n-m},\cdots,\frac{m-1}{n-m})^{\mathsf{T}}$.
\end{proof}
\subsection*{Proofs for the General Regime}
Now we take a close look at the M-SGD algorithm. Consider the gradient descent and its continuous version given as 
\begin{align}
    \label{gd}\tilde x_{k+1}&=\tilde x_k-\gamma \nabla g(\tilde x_k),\\
    \label{gdcont} d\tilde{X}_t&=-\nabla g(\tilde{X}_t)dt.
\end{align}
We start by showing a result about \eqref{gd} and \eqref{gdcont}.

\begin{lemma}\label{gdlemma}
	Under assumption~\ref{assm1}, for the gradient descent algorithm $\tilde x_k$ defined by~\eqref{gd} and its continuous version $\tilde{X}_t$  defined by~\eqref{gdcont}, we have
	\begin{align*}
	\left|\tilde{x}_k-\tilde{X}_{k\gamma}\right|\le C_1k\gamma \, \left(1+L\gamma\right)^{k}.
	\end{align*}
\end{lemma}
\begin{proof}
	Note that $\tilde{X_0}=\tilde{x_0}$ by assumption~\ref{assm1}. Let $t \in [0,\gamma]$. Note that 
	\begin{align*}
	\tilde X_t&=\tilde X_0-\int_{0}^{t}\nabla g(\tilde X_s) ds;\\
	\tilde x_1&=\tilde{x_0}-\gamma\nabla g(\tilde{x_0}).
	\end{align*}
	Hence, we have
	\begin{align*}
	\left|\tilde{X_t}-\tilde{x_1}\right|&\le \int_{0}^{t}\left|\nabla g(\tilde{X_s})-\nabla g(\tilde{x_0})\right|+\left(\gamma-t\right)\left|\nabla g(\tilde x_0)\right|\\
	&\le L\int_{0}^{t}\left|\tilde X_s-\tilde x_0\right|+\gamma \left|\nabla g(\tilde{x_0})\right|.
	\end{align*}
	Therefore
	\begin{align*}
	\left|\tilde{X_t}-\tilde{x_1}\right|& \le L\int_{0}^{t}\left|\tilde X_s-\tilde x_1\right|+Lt\left|\tilde{x_1}-\tilde{x_0}\right|+\gamma \left|\nabla g(\tilde{x_0})\right|\\
	&\le L\int_{0}^{t}\left|\tilde X_s-\tilde x_1\right|+L\gamma^2\left|\nabla g(\tilde{x_0})\right|+\gamma \left|\nabla g(\tilde{x_0})\right|\\
	&\le \left|\nabla g(\tilde{x_0})\right|\gamma \left(1+L\gamma\right)+ L\int_{0}^{t}\left|\tilde X_s-\tilde x_1\right|.
	\end{align*}
	Using the Gronwall Lemma, we get
	\begin{align*}
	\left|\tilde{X_t}-\tilde{x_1}\right|&\le \left|\nabla g(\tilde{x_0})\right|\gamma \left(1+L\gamma\right) e^{L\gamma}.
	\end{align*}
	Note that as we are in the regime $0\le k\le [T/\gamma]$, we can bound $e^{L\gamma}$ by $e^T$ where as stated before $T$ is fixed. Let the induction step in $t \in [(k-1)\gamma,k\gamma]$ hold as 
	\begin{align*}
	\left|\tilde X_t-\tilde x_k\right|\le \left|\nabla g(\tilde{x_0})\right|k\gamma e^{kL\gamma}\left(1+\gamma L\right)^{k}.
	\end{align*}
	Let $t \in [k\gamma,(k+1)\gamma]$. Note that
	\begin{align*}
	\left|\tilde X_t-\tilde x_{k+1}\right|& \le \left|\tilde X_{k\gamma}-\tilde x_k\right|+L\int_{k\gamma}^{t}\left|\tilde X_s-\tilde x_k\right| ds+\left((k+1)\gamma-t\right)\left|\nabla g(\tilde x_k)\right|\\
	&\le \left|\nabla g(\tilde{x_0})\right|k\gamma e^{kL\gamma}\left(1+\gamma L\right)^{k}+L\int_{k\gamma}^{t}\left|\tilde X_s-\tilde x_{k+1}\right|+\left(1+L\gamma\right)\gamma\left|\nabla g(\tilde x_k)\right|.
	\end{align*}
	Also, noting that $\left|\nabla g(\tilde{x}_k)\right|\le \left(1+\gamma L\right)^k \left|\nabla g(\tilde{x}_0)\right|$ as 
    \[\left|\nabla g(\tilde {x}_{k+1})\right|\le L\,\left|\tilde{x}_{k+1}-\tilde{x}_k\right|+\left|\nabla g(\tilde{x}_k)\right|\le L\gamma \left|\nabla g(\tilde{x}_k)\right| +\left|\nabla g(\tilde{x}_k)\right|=\left(1+L\gamma\right)\left|\nabla g(\tilde{x}_k)\right|\]
    and rearranging some terms we can see that 
	\begin{align*}
	\left|\tilde X_t-\tilde x_{k+1}\right|\le \left|\nabla g(\tilde{x_0})\right|\left(k+1\right)\gamma e^{kL\gamma}\left(1+\gamma L\right)^{k+1} +L\int_{k\gamma}^{t}\left|\tilde X_s-\tilde x_{k+1}\right|.
	\end{align*}
	Using the Gronwall inequality the induction step is completed. This completes the proof.
\end{proof}
\begin{lemma}\label{easybound1}
	For any $a,b\ge 0, \,x\ge 1$, we have
	\begin{align*}
	\left(1+\frac{a}{x}+\frac{b}{x^2}\right)^x  \le \exp\left(a+b+1\right).
	\end{align*}
	
\end{lemma}
The proof of Lemma~\ref{easybound1} is trivial; therefore we skip the proof. 
\begin{lemma}\label{moderatediffsq}
	Under assumptions~\ref{assm1}-\ref{assm5}, we have for the algorithms described in \eqref{sgd} and \eqref{gd},
	\begin{align*}
	\mathbb{E}\left|x_{k+1}-\tilde{x}_{k+1}\right|^2 \le \tilde{C}^*_1 \frac{\gamma^2}{m}+\tilde{C}^*_2 \frac{k}{m},
	\end{align*}
	where
	\begin{align*}
	\tilde{C}^*_1&=	e^{\pi^2/6}\exp\left(1+2TL+T^2L^2+2TL_1p+2T^2LL_1p+T^2p^2L^2_1\right)\left|\left|\sigma(x_0)\right|\right|^2_F\\
	&\quad \text{and}\\
	\tilde{C}^*_2&=\frac{\tilde{C}^*_1}{\left\|\sigma(x_0)\right\|_F^2}\left(T^2+1\right) \left(\ 2\, L^2_1\, p\, C^2_1 T^2 e^{2LT} + 2\sup_{0\le t\le T} \left|\left|\sigma(\tilde{X}_t)\right|\right|^2_F\right).
	\end{align*}	
\end{lemma}
\begin{proof}[Proof of Lemma \ref{moderatediffsq}] Noe that
	\begin{align*}
	x_{k+1}-\tilde{x}_{k+1}&=\left(x_{k}-\tilde{x}_{k}\right)-\gamma\left(\nabla g(x_k)-\nabla g(\tilde{x}_k)\right)+\sqrt{\frac{\gamma}{m}}\sigma(x_k)\sqrt\gamma \xi_{k+1}.
	\end{align*}
	This implies that
	\begin{align*}
	\left|x_{k+1}-\tilde{x}_{k+1}\right|&\le \left|x_{k}-\tilde{x}_{k}\right|+\gamma\left|\nabla g(x_k)-\nabla g(\tilde{x}_k) \right|+\sqrt{\frac{\gamma}{m}}\left|\sigma(x_k)\sqrt\gamma \xi_{k+1}\right|.
	\end{align*}
	Using the fact that $\nabla g$ is Lipschitz, we have
	\begin{align*}
	\left|x_{k+1}-\tilde{x}_{k+1}\right|&\le \left|x_{k}-\tilde{x}_{k}\right|+\gamma L\left|x_k-\tilde{x}_k \right|+\frac{\gamma}{\sqrt m}\left|\sigma(x_k)\xi_{k+1}\right|;\\
	\left|x_{k+1}-\tilde{x}_{k+1}\right|&\le \left(1+\gamma L\right)\left|x_k-\tilde{x}_k \right|+\frac{\gamma}{\sqrt m}\left|\left(\sigma(x_k)-\sigma(\tilde{x}_k)\right)\xi_{k+1}\right|+\frac{\gamma}{\sqrt m}\left|\sigma(\tilde{x}_k)\xi_{k+1}\right|.
	\end{align*}
	Also, using the fact $\left|\left(\sigma(x_k)-\sigma(\tilde{x}_k)\right)\xi_{k+1}\right|\le \left|\sigma(x_k)-\sigma(\tilde{x}_k)\right|_2\left|\xi_{k+1}\right|$ and\\ assumption~\ref{assm3} above, we have
	\begin{align*}
	\left|x_{k+1}-\tilde{x}_{k+1}\right|&\le \left(1+\gamma L\right)\left|x_k-\tilde{x}_k \right|+\frac{\gamma}{\sqrt m}\left|\left(\sigma(x_k)-\sigma(\tilde{x}_k)\right)\xi_{k+1}\right|+\frac{\gamma}{\sqrt m}\left|\sigma(\tilde{x}_k)\xi_{k+1}\right|;\\
	\left|x_{k+1}-\tilde{x}_{k+1}\right|&\le \left(1+\gamma L\right)\left|x_k-\tilde{x}_k \right|+\frac{\gamma}{\sqrt m}L_1 \left|x_k-\tilde{x}_k\right|\left|\xi_{k+1}\right|+\frac{\gamma}{\sqrt m}\left|\sigma(\tilde{x}_k)\xi_{k+1}\right|\\
	&\le \left[\left(1+\gamma L\right) + \frac{\gamma}{\sqrt{m}}L_1\left|\xi_{k+1}\right|\right]\left|x_k-\tilde{x}_k\right|+\frac{\gamma}{\sqrt{m}}\left|\sigma(\tilde{x}_k)\xi_{k+1}\right|.
	\end{align*}
	Square the above and then use Jensen's inequality to see
	\begin{align*}
	&\left|x_{k+1}-\tilde{x}_{k+1}\right|^2\le \left(1+\frac{1}{k^2}\right)\left[\left(1+\gamma L\right) + \frac{\gamma}{\sqrt{m}}L_1\left|\xi_{k+1}\right|\right]^2 \left|x_k-\tilde{x}_k \right|^2\\
	&\quad \quad \quad \quad \quad \quad\quad \quad \quad \quad +\frac{\left(1+k^2\right)\gamma^2}{m}\left|\sigma(\tilde{x}_k)\xi_{k+1}\right|^2.
	\end{align*}
	By taking expectation, we get
	\begin{align*}
	&\mathbb{E}\left|x_{k+1}-\tilde{x}_{k+1}\right|^2\\
	&\le \left(1+\frac{1}{k^2}\right)\mathbb{E}\left[\left(1+\gamma L\right) + \frac{\gamma}{\sqrt{m}}L_1\left|\xi_{k+1}\right|\right]^2 \mathbb{E}\left|x_k-\tilde{x}_k \right|^2+\frac{\left(1+k^2\right)\gamma^2}{m}\left|\sigma(\tilde{x}_k)\xi_{k+1}\right|^2\\
	& \le \left[\left(1+\gamma L\right)^2+\frac{2\gamma\,\sqrt{p}}{\sqrt{m}}L_1 \left(1+\gamma L\right)+\frac{\gamma^2}{m}p\,L^2_1\right]^k\Bigg\{\left[\prod_{j=1}^{k}\left(1+\frac{1}{j^2}\right)\right]\mathbb{E}\left|x_1-\tilde{x}_1\right|^2\\
	&\quad +\sum_{j=1}^{k-1}\prod_{i=1}^{j}\left(1+\frac{1}{\left(k-i+1\right)^2}\right)\frac{\gamma^2\left((k-j)^2+1\right)}{m}\left|\left|\sigma(\tilde{x}_{k-j})\right|\right|^2_F\Bigg\}\\
	&\quad +\frac{\gamma^2\left(k^2+1\right)}{m}\left|\left|\sigma(\tilde x_k)\right|\right|^2_F.
	\end{align*}
	Using the fact that the arithmetic mean of $n$ positive real numbers is always greater than the geometric mean of the same real numbers, for any $k \in \mathbb{N}$, we obtain
	\begin{align*}
	\prod_{j=1}^{k}\left(1+\frac{1}{j^2}\right) \le \left(1+\frac{1}{k}\sum_{j=1}^{k}\frac{1}{j^2}\right)^k.
	\end{align*}
	Also $\sum_{j=1}^{k}\frac{1}{j^2} \le \sum_{j=1}^{\infty} \frac{1}{j^2}=\frac{\pi^2}{6}$. Therefore we have
	\begin{align*}
	\prod_{j=1}^{k}\left(1+\frac{1}{j^2}\right) &\le \left(1+\frac{\pi^2}{6k}\right)^k\\
	&\le e^{\pi^2/6}.
	\end{align*}
	Also, note that
	\begin{align*}
	\mathbb{E}\left|x_1-\tilde{x}_1\right|^2=\frac{\gamma^2}{m}\left|\left|\sigma(x_0)\right|\right|^2_F.
	\end{align*}
	Therefore, using Lemma~\ref{easybound1} the first term is less than
	\begin{align*}
	e^{\pi^2/6}\exp\left(1+2TL+T^2L^2+2TL_1\,\sqrt{p}+2T^2LL_1\,\sqrt{p}+T^2p\,L^2_1\right)\frac{\gamma^2}{m} \,\mathbb{E}\left|\left|\sigma(x_0)\right|\right|^2_F=\tilde{C}^*_1 \frac{\gamma^2}{m}.
	\end{align*}
	Let us define 
	\[\tilde{L}^*=\exp\left(1+2TL+T^2L^2+2TL_1\,\sqrt{p}+2T^2LL_1\,\sqrt{p}+T^2p\,L^2_1\right).\]
	It can be seen that the second term is bounded by 
	\begin{align*}
	& \tilde{L}^*\,\frac{T^2+1}{m}\sup_{0\le k\le T/\gamma} \left|\left|\sigma(\tilde x_k)\right|\right|^2_F \cdot k\prod_{j=1}^{k}\left(1+\frac{1}{j^2}\right)\\
	&\le e^{\pi^2/6}\frac{\left(T^2+1\right)k}{m}\sup_{0\le k\le T/\gamma} \left|\left|\sigma(\tilde x_k)\right|\right|^2_F.
	\end{align*}
	Observe
	\begin{align*}
	\left(\Tr \sigma^2(\tilde{x}_{k})\right)^{1/2} &=\left|\left|\sigma(\tilde{x}_{k})\right|\right|_F\\
	& \le \left|\left|\sigma(\tilde{x}_{k})-\sigma(\tilde{X}_{k\gamma})\right|\right|_F+\left|\left|\sigma(\tilde{X}_{k\gamma})\right|\right|_F\\
	&\le  L_1 \sqrt{p}\left|\tilde{x}_k-\tilde{X}_{k\gamma}\right|+\sup_{0\le t\le T} \left|\left|\sigma(\tilde{X}_t)\right|\right|_F\\
	& \le  L_1 \sqrt{p} \ C_1 k \gamma \left(1+L\gamma\right)^{k} + \sup_{0\le t\le T} \left|\left|\sigma(\tilde{X}_t)\right|\right|_F.
	\end{align*}
	The second inequality follows from the relation between spectral norm and Frobenius norm and the last one follows from the fact that the discretized version of gradient descent is only an order of the step size away from its continuous counterpart.
	Thus,
	\begin{align*}
	\sup_{0\le k\le T/\gamma} \left|\left|\sigma(\tilde x_k)\right|\right|^2_F \le 2 L^2_1 p \ C^2_1 T^2 e^{2\,LT} + 2\sup_{0\le t\le T} \left|\left|\sigma(\tilde{X}_t)\right|\right|^2_F.
	\end{align*}
	Therefore the second terms is less than
	\begin{align*}
	\tilde{L}^*\,e^{\pi^2/6}\frac{k\left(T^2+1\right)}{m} \left(2 \, L^2_1\,p\,C^2_1 T^2 e^{2\,LT} + 2\sup_{0\le t\le T} \left|\left|\sigma(\tilde{X}_t)\right|\right|^2_F\right)=\tilde{C}^*_{2} \frac{k}{m}.
	\end{align*}
	As a result
	\begin{align*}
	E\left|x_{k+1}-\tilde{x}_{k+1}\right|^2 \le \tilde{C}^*_1 \frac{\gamma^2}{m}+\tilde{C}^*_2 \frac{k}{m}.
	\end{align*}
\end{proof}
\noindent\textbf{Note:} Using Lemma~\ref{easybound1}, we have the rate in the previous lemma as $\frac{k}{m}$. This is because all the other terms are bounded in the regime $\gamma K=T$.
%%%%%
Next we show \eqref{msgd} and \eqref{gd} are close in the Wasserstein distance.
%%%{Need another lemma before we can show this.}
To show the above statement we need one more lemma.
\begin{lemma}
	
	Under assumptions~\ref{assm1}-\ref{assm5} , we have 
	\begin{align*}
	\mathbb{E}\left|\sum_{i=1}^{n}w_{i,k}\left(\nabla l(x_{n,k},u_{i,k})- \nabla l(\tilde{x}_k,u_{i,k})\right)\right|^2 \le 2\left(\frac{n-m}{mn}+1\right)\mathbb{E}\left(h_1^2(u)\right)\mathbb{E}\left|x_{k,n}-\tilde{x}_k\right|^2.
	\end{align*}
\end{lemma}
\begin{proof}
	We start with bounding the $\mathbb{L}_2$ norm.
	\begin{align*}
	&\mathbb{E}\left|\sum_{i=1}^{n}w_{i,k}\left(\nabla l(x_{n,k},u_{i,k})- \nabla l(\tilde{x}_k,u_{i,k})\right)\right|^2 \\& \le \mathbb{E}\Bigg(\left|\sum_{i=1}^{n}\left(w_{i,k}-\frac{1}{n}\right)(\nabla l(x_{n,k},u_{i,k})- \nabla l(\tilde{x}_k,u_{i,k}))\right|\\
	&\quad +\left|\sum_{i=1}^{n}\frac{1}{n}(\nabla l(x_{n,k},u_{i,k})- \nabla l(\tilde{x}_k,u_{i,k}))\right|\Bigg)^2.
	\end{align*}
	Thus
	\begin{align*}
	&\mathbb{E}\left|\sum_{i=1}^{n}w_{i,k}\left(\nabla l(x_{n,k},u_{i,k})- \nabla l(\tilde{x}_k,u_{i,k})\right)\right|^2	\\ &\le 2 \, \mathbb{E}\left|\sum_{i=1}^{n}\left(w_{i,k}-\frac{1}{n}\right)(\nabla l(x_{n,k},u_{i,k})- \nabla l(\tilde{x}_k,u_{i,k}))\right|^2\\
	&\quad + 2 \, \mathbb{E}\left(\frac{1}{n}\sum_{i=1}^{n}h_1(u_i)\left|x_{n,k}-\tilde{x}_k\right|\right)^2.
	\end{align*}
	Hence
	\begin{align*}
	&\mathbb{E}\left|\sum_{i=1}^{n}w_{i,k}\left(\nabla l(x_{n,k},u_{i,k})- \nabla l(\tilde{x}_k,u_{i,k})\right)\right|^2	\\ &\le 2 \, \mathbb{E}\Bigg(\sum_{i=1}^{n}\left(w_{i,k}-\frac{1}{n}\right)^2\left|\nabla l(x_{n,k},u_{i,k})-\nabla l(\tilde{x}_k,u_{i,k})\right|^2\\
	&\quad +\sum_{i\ne j}\left(w_{i,k}-\frac{1}{n}\right)\left(w_{j,k}-\frac{1}{n}\right)\\
	&\quad\quad\left(\nabla l(x_{n,k},u_{i,k})-\nabla l(\tilde{x}_k,u_{i,k})\right)^{\mathsf{T}}\left(\nabla l(x_{n,k},u_{j,k})-\nabla l(\tilde{x}_k,u_{j,k})\right)\Bigg)\\ 
	&\quad+2 \, \mathbb{E}\left(\frac{1}{n}\sum_{i=1}^{n}h_1(u_{i,k})\right)^2 \mathbb{E}\left|x_{n,k}-\tilde{x}_k\right|^2.
	\end{align*}
	Therefore
	\begin{align*}
	&\mathbb{E}\left|\sum_{i=1}^{n}w_{i,k}\left(\nabla l(x_{n,k},u_{i,k})- \nabla l(\tilde{x}_k,u_{i,k})\right)\right|^2\\&\le 2 \Bigg(\sum_{i=1}^{n}\frac{n-m}{mn^2} \,\mathbb{E}\left|\nabla l(x_{n,k},u_{i,k})-\nabla l(\tilde{x}_k,u_{i,k})\right|^2\\
	&\quad - \sum_{i\ne j} \frac{n-m}{mn^2(n-1)} \,\mathbb{E}\left|\nabla g(x_{n,k})-\nabla g(\tilde{x}_k)\right|^2\Bigg)\\
	&\quad + 2\,\mathbb{E}(h_1(u))^2\,\mathbb{E}\left|x_{n,k}-\tilde{x}_k\right|^2 
	\end{align*}
	using the definition of the weights $w_i$'s and Jensen's inequality.
	This concludes 
	\begin{align*}
	&\mathbb{E}\left|\sum_{i=1}^{n}w_{i,k}\left(\nabla l(x_{n,k},u_{i,k})- \nabla l(\tilde{x}_k,u_{i,k})\right)\right|^2 	\\&\le 2 \Bigg(\sum_{i=1}^{n}\frac{n-m}{mn^2} \,\mathbb{E}\left|\nabla l(x_{n,k},u_{i,k})-\nabla l(\tilde{x}_k,u_{i,k})\right|^2\Bigg)\\
	&\quad +2\, \,\mathbb{E}(h_1(u))^2\mathbb{E}\left|x_{n,k}-\tilde{x}_k\right|^2.
	\end{align*}
	Which implies
	\begin{align*}
	&\mathbb{E}\left|\sum_{i=1}^{n}w_{i,k}\left(\nabla l(x_{n,k},u_{i,k})- \nabla l(\tilde{x}_k,u_{i,k})\right)\right|^2 \\&\le 2 \Bigg(\sum_{i=1}^{n}\frac{n-m}{mn^2}\mathbb{E}\left(h_1^2(u)\right)\mathbb{E}\left|x_{n,k}-\tilde{x}_k\right|^2\Bigg)\\
	&\quad +2\, \mathbb{E}(h_1(u))^2\,\mathbb{E}\left|x_{k,n}-\tilde{x}_k\right|^2\\
	& \le 2\left(\frac{n-m}{mn}+1\right) \mathbb{E}\left(h_1^2(u)\right) \mathbb{E}\left|x_{n,k}-\tilde{x}_k\right|^2 .
	\end{align*}
\end{proof}
\begin{lemma}\label{jingggle}
	With $x_{n,k}$ and $\tilde x_k$ as defined in \eqref{msgd} and \eqref{gd}, under assumptions~\ref{assm1}-\ref{assm5}, we have 
	\begin{align*}
	\mathbb{E}\left|x_{n,k+1}-\tilde x_{k+1}\right|^2 &\le K^*_1 \,\frac{3^k\gamma}{m},
	\end{align*}
	where 
	\begin{align*}
	K^{*}_1=T\cdot \exp\left(1+6T^2 \mathbb{E}\left(h_1^2(u)\right)\right) \cdot \left(2p^2L^2_1 C^2_1 T^2 e^{2TL}+2\sup_{0\le t\le T}||\sigma(\tilde{X}_t)||^2_F\right).
	\end{align*}
\end{lemma}
\begin{proof}
	Recall the definitions of \eqref{msgd} and \eqref{gd}.
	Then
	\begin{align*}
	x_{n,k+1}-\tilde x_{k+1}&=x_{n,k}-\tilde x_{k}-\gamma \sum_{i=1}^{n}w_{i,k}\left(\nabla l\left(x_{n,k},u_{i,k}\right)-g(\tilde x_k)\right)\\
	&=x_{n,k}-\tilde x_{k}-\gamma\sum_{i=1}^{n}w_{i,k}\left(\nabla l\left(x_{n,k},u_{i,k}\right)-\nabla l(\tilde{x}_k,u_{i,k}\right)-\gamma\sum_{i=1}^{n}\left(\nabla l(\tilde{x}_k,u_{i,k}-\gamma g(\tilde x_k)\right).
	\end{align*}
	This implies that 
	\begin{align*}
	&\left|	x_{n,k+1}-\tilde x_{k+1}\right|\\ & \le \left|x_{n,k}-\tilde x_{k}\right| +\gamma \left|\sum_{i=1}^{n}w_{i,k}\left(\nabla l(x_{n,k},u_i)-\nabla l(\tilde{x}_k,u_i)\right)\right|+\gamma \left|\sum_{i=1}^{n}w_{i,k}\left(\nabla l(\tilde{x}_k,u_i)-\nabla g(\tilde{x}_k)\right)\right|.
	\end{align*}
	Square both sides of the above and  use Jensen's inequality to get
	\begin{align*}
	&\left|	x_{n,k+1}-\tilde x_{k+1}\right|^2\\ &\le 3  \left|x_{n,k}-\tilde x_{k}\right|^2
+3\gamma^2 \left|\sum_{i=1}^{n}w_{i,k}\left(\nabla l(x_{n,k},u_i)-\nabla l(\tilde{x}_k,u_i)\right)\right|^2+3\gamma^2 \left|\sum_{i=1}^{n}w_{i,k}\left(\nabla l(\tilde{x}_k,u_i)-\nabla g(\tilde{x}_k)\right)\right|^2.
	\end{align*}
	Thus, taking expectation, we have
	\begin{align*}
	&\mathbb{E}\left|x_{n,k+1}-\tilde x_{k+1}\right|^2\\ &\le 3 \,  \mathbb{E}\left|x_{n,k}-\tilde x_{k}\right|^2+3\gamma^2 \mathbb{E}\left|\sum_{i=1}^{n}w_{i,k}\left(\nabla l(x_{n,k},u_i)-\nabla l(\tilde{x}_k,u_i)\right)\right|^2\\&\quad +3\gamma^2 \mathbb{E}\left|\sum_{i=1}^{n}w_{i,k}\left(\nabla l(\tilde{x}_k,u_i)-\nabla g(\tilde{x}_k)\right)\right|^2.
	\end{align*}
	This implies that
	\begin{align*}
	&\mathbb{E}\left|x_{n,k+1}-\tilde x_{k+1}\right|^2\\
	&\le 3 \,\mathbb{E}\left|x_{n,k}-\tilde x_{k}\right|^2 + 3\gamma^2 2\left(\frac{n-m}{mn}+1\right)\mathbb{E}\left(h_1^2(u)\right)\mathbb{E}\left|x_{n,k}-\tilde x_{k}\right|^2+\frac{3\gamma^2}{m}||\sigma(\tilde{x}_k)||^2_F\\
	&=3\left[1+2\gamma^2\left(\frac{n-m}{mn}+1\right)\mathbb{E}\left(h_1^2(u)\right)\right]\mathbb{E}\left|x_{n,k}-\tilde x_{k}\right|^2+\frac{3\gamma^2}{m}||\sigma(\tilde{x}_k)||^2_F.
	\end{align*}
	Here the last inequality follows from the previous lemma.
	Continuing from the last line, we get
	\begin{align*}
	&\mathbb{E}\left|	x_{n,k+1}-\tilde x_{k+1}\right|^2 \\&\le \sum_{j=0}^{k} 3^{j+1}\left[1+ 2\gamma^2\left(\frac{n-m}{mn}+1\right)\mathbb{E}\left(h_1^2(u)\right)\right]^{j} \frac{\gamma^2}{m}||\sigma(\tilde{x}_{k-j})||^2_F\\
	&\le\frac{\gamma^2}{m}\ \left(k+1\right) \ 3^{k+1} \left[1 + 2\gamma^2\left(\frac{n-m}{mn}+1\right)\mathbb{E}\left(h_1^2(u)\right)\right]^{k} \sup_{0\le k \le [\frac{T}{\gamma}]} ||\sigma(\tilde{x}_k)||^2_F\\
	& \le \frac{\gamma^2}{m}\ \left(k+1\right) \ 3^{k+1}  \left[1 + 2\gamma^2\left(\frac{n-m}{mn}+1\right)\mathbb{E}\left(h_1^2(u)\right)\right]^{k}\\
	&\quad \left[2p^2L^2_1C^2_1k^2\gamma^2\left(1+L\gamma\right)^{2k}+2\sup_{0\le t\le T}||\sigma(\tilde{X}_t)||^2_F\right].
	\end{align*}
 	From Lemma~\ref{easybound1}, we see that 
	\begin{align*}
	\left[1+ 2\gamma^2\left(\frac{n-m}{mn}+1\right)\mathbb{E}\left(h_1^2(u)\right)\right]^{k} \le \exp\left(1+6T^2 \,\mathbb{E}\left(h_1^2(u)\right)\right). 
	\end{align*}
	We use the fact $\left(\frac{n-m}{mn}+1\right)\le 3$ in the above bound.
	Also, using the fact $K\gamma=T$, we have
	\begin{align*}
	\left[2p^2L^2_1C^2_1k^2\gamma^2\left(1+L\gamma\right)^{2k}+2\sup_{0\le t\le T}||\sigma(\tilde{X}_t)||^2_F\right]\le 2p^2L^2_1 C^2_1 T^2 e^{2TL}+2\sup_{0\le t\le T}||\sigma(\tilde{X}_t)||^2_F.
	\end{align*}
	Hence, 
	\begin{align*}
	\mathbb{E}\left|	x_{n,k+1}-\tilde x_{k+1}\right|^2 \le K^{*}_1 \frac{3^k \gamma}{m},
	\end{align*}
	where 
	\begin{align*}
	K^{*}_1=T\cdot \exp\left(1+6T^2 \mathbb{E}(h_1^2(u))\right) \cdot \left(2p^2L^2_1 C^2_1 T^2 e^{2TL}+2\sup_{0\le t\le T}||\sigma(\tilde{X}_t)||^2_F\right).
	\end{align*}
	Thus we conclude the proof.
\end{proof}
\noindent\textbf{Note:}
As we can see that 
the rate here is $4^k\gamma/m$ with $K^{*}_1$ as a constant dependent on $T,L,L_1,p$.
%Using the above results, we get, 
\begin{prop}\label{thmgg1}
	Suppose assumptions~\ref{assm1}-\ref{assm5} hold. Recall $x_{n,k}$ and $x_k$ from equations~\eqref{msgd} and~\eqref{sgd}. Then for any $k \le K= \left[T/\gamma\right]$ with $m\ge 3^{K}$, we have
	\begin{align*}
	\mathbb{E}\left|x_{n,k}-x_{k}\right|^2 
	&\le K_1\cdot \gamma,
	\end{align*}
	where $K_1$ is a constant dependent only on $T,L,L_1 \ \text{and} \ p$.
\end{prop}
\begin{proof}[Proof of Theorem~\ref{thmgg1}]
	Combining Lemma~\ref{moderatediffsq} and Lemma~\ref{jingggle} and using the\\
    Cauchy-Schwarz inequality, we have
	\begin{align*}
	\mathbb{E}\left|x_{n,k+1}-x_{k+1}\right|^2 &\le 2 \, \mathbb{E}\left|x_{n,k+1}-\tilde{x}_{k+1}\right|^2 +2 \, \mathbb{E}\left|x_{k+1}-\tilde{x}_{k+1}\right|^2\\
	&\le 2\tilde{C}^*_1\frac{\gamma^2}{m}+2\tilde{C}^*_2 \frac{k}{m} +2 K^{*}_1\frac{3^k\gamma}{m} .
	\end{align*}
	This gives us
	\begin{align*}
	\mathbb{E}\left|x_{n,k+1}-x_{k+1}\right|^2
	& \le K_1 \frac{3^k\gamma}{m},
	\end{align*}
	where   $K_1=2\tilde{C}^*_1 +2\tilde{C}^*_2\frac{1}{T}+2K^*_1$.
	Note that $\gamma/3^k<1$ and $\frac{k}{3^k\gamma}=\frac{k^2}{3^k T}<\frac{1}{T}$ since $k^2<3^k$ for all $k \ge 1$. 
\end{proof}
Now, we shall address the problem with only positive weights.
\begin{lemma}\label{jingggle_1}
	Let $x_{n,k}$ and $\tilde x_k$ be as defined in \eqref{msgd} and \eqref{gd}. Then, under assumptions~\ref{assm1}-\ref{assm5}, with $w_{i,k} \ge 0$ for all $k\le K$, we have 
	\begin{align*}
	\mathbb{E}\left|x_{n,k+1}-\tilde x_{k+1}\right|^2 &\le D^*_1\,\frac{\gamma^2}{m}+D^*_2\,\frac{k}{m},
	\end{align*}
	where 
	\begin{align*}
	D^{*}_1&=\exp\left(1+2\,T\, \mathbb{E}(h_1(u))+T^2\, Var(h_1(u))\right)\, e^{\pi^2/6}\,\left\|\sigma(x_0)\right\|^2_F\\
	D^*_2&=\exp\left(1+2\,T\, \mathbb{E}(h_1(u))+T^2\, Var(h_1(u))\right)\, e^{\pi^2/6}\,\left(T^2+1\right) \left(2 \, L^2_1\,p\,C^2_1 T^2 e^{2\,LT} + 2\sup_{0\le t\le T} \left|\left|\sigma(\tilde{X}_t)\right|\right|^2_F\right).
	\end{align*}
\end{lemma}
\begin{proof}[Proof of Lemma~\ref{jingggle_1}]
   	Recall the definitions of \eqref{msgd} and \eqref{gd}.
	Then
	\begin{align*}
	x_{n,k+1}-\tilde x_{k+1}&=x_{n,k}-\tilde x_{k}-\gamma \sum_{i=1}^{n}w_{i,k}\left(\nabla l\left(x_{n,k},u_{i,k}\right)-g(\tilde x_k)\right)\\
	&=x_{n,k}-\tilde x_{k}-\gamma\sum_{i=1}^{n}w_{i,k}\left(\nabla l\left(x_{n,k},u_{i,k}\right)-\nabla l(\tilde{x}_k,u_{i,k}\right)-\gamma\sum_{i=1}^{n}\left(\nabla l(\tilde{x}_k,u_{i,k}-\gamma g(\tilde x_k)\right).
	\end{align*}
	This implies that 
	\begin{align*}
	&\left|	x_{n,k+1}-\tilde x_{k+1}\right|\\ & \le \left|x_{n,k}-\tilde x_{k}\right| +\gamma \left|\sum_{i=1}^{n}w_{i,k}\left(\nabla l(x_{n,k},u_i)-\nabla l(\tilde{x}_k,u_i)\right)\right|+\gamma \left|\sum_{i=1}^{n}w_{i,k}\left(\nabla l(\tilde{x}_k,u_i)-\nabla g(\tilde{x}_k)\right)\right|\\
	&\le \left(1+\gamma \sum_{i=1}^{n} w_i \, h(u_i)\right)\left|x_{n,k}-\tilde{x}_k\right|+\gamma \left|\sum_{i=1}^{n}w_{i,k}\left(\nabla l(\tilde{x}_k,u_i)-\nabla g(\tilde{x}_k)\right)\right|.
	\end{align*}
 Therefore
 \begin{align*}
    	&\mathbb{E}\left|	x_{n,k+1}-\tilde x_{k+1}\right|^2\\
    	&\le \left(1+\frac{1}{k^2}\right)\left[1+2\,\gamma \, \mathbb{E}(h(u_1))+\gamma^2\, \frac{1}{m}\left(1-\frac{m}{n}\right)Var(h(u_1))\right]\mathbb{E}\left|x_{n,k}-\tilde{x}_k\right|^2+\gamma^2 \left(k^2+1\right)\frac{1}{m}\left\|\sigma(\tilde{x}_k)\right\|^2_F\\
    	& \le \left[1+2\,\gamma \, \mathbb{E}(h(u_1))+\gamma^2\, \frac{1}{m}Var(h(u_1))\right]^k\Bigg\{\left[\prod_{j=1}^{k}\left(1+\frac{1}{j^2}\right)\right]\mathbb{E}\left|x_1-\tilde{x}_1\right|^2\\
	&\quad +\sum_{j=1}^{k-1}\prod_{i=1}^{j}\left(1+\frac{1}{\left(k-i+1\right)^2}\right)\frac{\gamma^2\left((k-j)^2+1\right)}{m}\left|\left|\sigma(\tilde{x}_{k-j})\right|\right|^2_F\Bigg\}\\
	&\quad +\frac{\gamma^2\left(k^2+1\right)}{m}\left|\left|\sigma(\tilde x_k)\right|\right|^2_F.
 \end{align*}
 Thus
 \begin{align*}
    	&\mathbb{E}\left|	x_{n,k+1}-\tilde x_{k+1}\right|^2\\
    	& \le \exp\left(1+2\,T\, \mathbb{E}(h_1(u))+T^2\, Var(h_1(u))\right)\, e^{\pi^2/6}\, \frac{\gamma^2}{m}\left\|\sigma(x_0)\right\|^2_F\\
    	& \quad +\exp\left(1+2\,T\, \mathbb{E}(h_1(u))+T^2\, Var(h_1(u))\right)\, e^{\pi^2/6}\,\left(T^2+1\right)\\
    	&\quad \quad \left(2 \, L^2_1\,p\,C^2_1 T^2 e^{2\,LT} + 2\sup_{0\le t\le T} \left|\left|\sigma(\tilde{X}_t)\right|\right|^2_F\right)\, \frac{k}{m}\\
    	&=D^*_1 \frac{\gamma^2}{m}+D^*_2\frac{k}{m}
 \end{align*}
 Hence the proof is completed.
\end{proof}
\begin{prop}\label{thmgg11}
	Suppose assumptions~\ref{assm1}-\ref{assm5} hold with $w_{i,k}\ge 0$ for all $k=1,2,\cdots,K$. Recall $x_{n,k}$ and $x_k$ from equations~\eqref{msgd} and~\eqref{sgd}. Then for any $k \le K= \left[T/\gamma\right]$, we have
	\begin{align*}
	\mathbb{E}\left|x_{n,k}-x_{k}\right|^2 
	&\le K_{11}\cdot \frac{\gamma^2}{m}+K_{12}\frac{k}{m},
	\end{align*}
	where $K_{11}$ is a constant dependent only on $T,L,L_1 \ \text{and} \ p$.
\end{prop}
\begin{proof}
    	Combining Lemma~\ref{moderatediffsq} and Lemma~\ref{jingggle_1} and using the\\
    Cauchy-Schwarz inequality, we have
	\begin{align*}
	\mathbb{E}\left|x_{n,k+1}-x_{k+1}\right|^2 &\le 2 \, \mathbb{E}\left|x_{n,k+1}-\tilde{x}_{k+1}\right|^2 +2 \, \mathbb{E}\left|x_{k+1}-\tilde{x}_{k+1}\right|^2\\
	&\le 2\tilde{C}^*_1\frac{\gamma^2}{m}+2\tilde{C}^*_2 \frac{k}{m} +2\tilde{D}^*_1\frac{\gamma^2}{m}+2\tilde{D}^*_2 \frac{k}{m}.
	\end{align*}
	In the last step, we use the fact that $m\ge K^2$. Hence 
	\begin{align*}
	  K_{11}&=2\,\left(C^*_1+D^*_1\right)\\
	  K_{12}&=2\,\left(C^*_2+D^*_2\right).
	\end{align*}
\end{proof}
Now we present few lemmas which shall be very important for our next steps.
\begin{lemma}\label{gsigmabound}
	Under assumptions~\ref{assm1}-\ref{assm5}, we have
	\begin{align*}
	\mathbb{E}\left|\nabla g(D_{k\gamma})\right|^2 \le \tilde{C_3}+\tilde{C}_1 \frac{\gamma^2}{m}+\tilde{C}_2 \frac{k}{m},
	\end{align*}
	where $\tilde C_3=4L^2C_1^2T^2e^{2LT}+4\sup_{0\le t\le T}\left|g(\tilde{X}_t)\right|^2$ and $\tilde{C}_1,\tilde{C}_2$ are $2\max(pL^2_1,L^2)\, \tilde{C}^*_1$ and $2\max(pL^2_1,L^2)\, \tilde{C}^*_2$, respectively, where $\tilde{C}^*_1,\tilde{C}^*_2$ are as defined in Lemma~\ref{moderatediffsq}.\\
	Also,
	\begin{align*}
	\mathbb{E}\left(\Tr \sigma^2(D_{k\gamma})\right) \le \tilde{C_4}+\tilde{C}_1 \frac{\gamma^2}{m}+\tilde{C}_2 \frac{k}{m},
	\end{align*} 
	where $\tilde{C}_4=4L^2C_1^2T^2e^{2LT}+4\sup_{0\le t\le T}\left|\left|\sigma(\tilde{X}_t)\right|\right|^2_F$.	
\end{lemma}
\begin{proof} Note that
	\begin{align*}
	\mathbb{E}\left|\nabla g(D_{k\gamma})\right|^2 &=\mathbb{E}\left|\nabla g(x_k)\right|^2\\
	&\le 2L^2 \ \mathbb{E}\left|x_k-\tilde{x}_k\right|^2+4L^2\left|\tilde{x}_k-\tilde{X}_{k\gamma}\right|^2+4\left|\nabla g(\tilde X_{k\gamma})\right|^2.
	\end{align*}
	Using Lemma~\ref{moderatediffsq} and Lemma~\ref{gdlemma}, we find that 
	\begin{align*}
	\mathbb{E}\left|\nabla g(D_{k\gamma})\right|^2&\le 4L^2C_1^2 k^2\gamma^2\left(1+L\gamma\right)^{2k}+4\sup_{0\le t\le T}\left|g(\tilde{X}_t)\right|^2\\
	&\quad + \tilde{C}_1 \frac{\gamma^2}{m}+\tilde{C}_2 \frac{k}{m}.
	\end{align*}
	This implies 
	\begin{align*}																					\mathbb{E}\left|\nabla g(D_{k\gamma})\right|^2			&\le 4L^2C_1^2T^2e^{2LT}+4\sup_{0\le t\le T}\left|g(\tilde{X}_t)\right|^2\\
	&\quad + \tilde{C}_1 \frac{\gamma^2}{m}+\tilde{C}_2 \frac{k}{m}\\
	&=\tilde{C_3}+\tilde{C}_1 \frac{\gamma^2}{m}+\tilde{C}_2 \frac{k}{m},
	\end{align*}
	where $\tilde C_3=4L^2C_1^2T^2e^{2LT}+4\sup_{0\le t\le T}\left|g(\tilde{X}_t)\right|^2$ and  $\tilde{C}_1,\tilde{C}_2$ are $2\max(pL^2_1,L^2)\, \tilde{C}^*_1$ and $2\max(pL^2_1,L^2)\, \tilde{C}^*_2$, respectively, where $\tilde{C}^*_1,\tilde{C}^*_2$ are as defined in Lemma~\ref{moderatediffsq}. We now do the same with $\sigma(\cdot)$. In fact, 
	\begin{align*}
	\mathbb{E}\left(\Tr \sigma^2(D_{k\gamma})\right)&=\mathbb{E}\left|\left|\sigma(x_k)\right|\right|^2_F\\
	& \ \le 2\, p\,L^2_1 \ \mathbb{E}\left|x_k-\tilde{x}_k\right|^2+2 \left|\left|\sigma(\tilde{x}_k)\right|\right|^2_F\\
	& \ \le 2\,p\,L^2_1 \ \mathbb{E}\left|x_k-\tilde{x}_k\right|^2+4 p L^2_1C_1^2 k^2 \gamma^2\left(1+L\gamma\right)^{2k}+4\sup_{0\le t\le T}\left|\left|\sigma(\tilde{X}_t)\right|\right|^2_F.
	\end{align*}
	Using this fact, and the bounds from Lemma~\ref{moderatediffsq} and Lemma~\ref{gdlemma}, we have
	\begin{align*}
	\mathbb{E}\left(\Tr \sigma^2(D_{k\gamma})\right)
	& \le 4pL^2_1C_1^2 k^2\gamma^2\left(1+L\gamma\right)^{2k}+4\sup_{0\le t\le T}\left|\left|\sigma(\tilde{X}_t)\right|\right|^2_F\\
	&\quad +\tilde{C}_1 \frac{\gamma^2}{m}+\tilde{C}_2 \frac{k}{m} \\
	&\le 4pL_1^2C_1^2T^2e^{2LT}+4\sup_{0\le t\le T}\left|\left|\sigma(\tilde{X}_t)\right|\right|^2_F\\
	&\quad + \tilde{C}_1 \frac{\gamma^2}{m}+\tilde{C}_2 \frac{k}{m}\\
	&\le \tilde{C}_4 + \tilde{C}_1 \frac{\gamma^2}{m}+\tilde{C}_2 \frac{k}{m},
	\end{align*} 
	where $\tilde{C}_4=4pL_1^2C_1^2T^2e^{2LT}+4\sup_{0\le t\le T}\left|\left|\sigma(\tilde{X}_t)\right|\right|^2_F$ and  $\tilde{C}_1,\tilde{C}_2$ are $2\max(pL^2_1,L^2)\, \tilde{C}^*_1$ and $2\max(pL^2_1,L^2)\, \tilde{C}^*_2$, respectively, where $\tilde{C}^*_1,\tilde{C}^*_2$ are as defined in Lemma~\ref{moderatediffsq}.
\end{proof}
\begin{lemma}\label{middlelemma}
	Under assumptions~\ref{assm1}-\ref{assm5}, for $t \in (k\gamma,(k+1)\gamma]$, we have 
	\begin{align*}
	\int_{k\gamma}^{t} \,\mathbb{E}\left|D_{s}-D_{k\gamma}\right|^2 \le K_{11}\gamma^3 + K_{12}\frac{\gamma^2}{m}
	\end{align*}
	where $K_{11}, K_{12}$ are functions dependent on $L,L_1,T,p$.
\end{lemma}
\begin{proof}[Proof of Lemma~\ref{middlelemma}]
	Using the definition of $D_{s}$, we have 	
	\begin{align*}
	\mathbb{E}\left|D_{s}-D_{k\gamma}\right|^2&= \,\mathbb{E}\left|-\nabla g(D_{k\gamma})\left(s-k\gamma\right)+\sqrt\frac{\gamma}{m}\sigma(D_{k\gamma})\left(B_s-B_{k\gamma}\right)\right|^2\\
	&\le 2 \,\mathbb{E}\left[\left|\nabla g(D_{k\gamma})\right|\left(s-k\gamma\right)\right]^2+2\frac{\gamma}{m} \,\mathbb{E}\left|\int_{s}^{k\gamma}\sigma(D_{k\gamma}) dB_l\right|^2\\
	&=2\left(s-k\gamma\right)^2 \mathbb{E}\left|\nabla g(D_{k\gamma})\right|^2+2\frac{\gamma}{m} \,\mathbb{E}\left(\Tr \sigma^2(D_{k\gamma})\right)\left(s-k\gamma\right).
	\end{align*}
	Hence, for $t \in [k\gamma,(k+1)\gamma]$,
	\begin{align*}
	\int_{k\gamma}^{t} \mathbb{E}\left|D_{s}-D_{k\gamma}\right|\le 2\int_{k\gamma}^{t} \left(s-k\gamma\right)^2 \mathbb{E}\left|\nabla g(D_{k\gamma})\right|^2+2\frac{\gamma}{m}\int_{k\gamma}^{t}\mathbb{E}\left|\left|\sigma(D_{k\gamma})\right|\right|^2_F\left(s-k\gamma\right).
	\end{align*}
	Now, using Lemma~\ref{gsigmabound} the first term is bounded by 
	\begin{align*}
	2\int_{k\gamma}^{t} \left(s-k\gamma\right)^2 \mathbb{E}\left|\nabla g(D_{k\gamma})\right|^2 
	&\le 2\int_{k\gamma}^{t}\left(s-k\gamma\right)^2 \left(\tilde{C}_3+\tilde{C}_1 \frac{\gamma^2}{m}+\tilde{C}_2 \frac{k}{m}\right)\\
	&\le \frac{2\gamma^3}{3}\tilde{C}_3+\frac{2}{3}\tilde{C}_1\frac{\gamma^5}{m}+\frac{2}{3}\tilde{C}_2 \frac{\gamma^2 T}{m}.
	\end{align*}	
	For the second term we can do the exact same thing. In fact,
	\begin{align*}
	2\frac{\gamma}{m}\int_{k\gamma}^{t}\mathbb{E}\left|\left|\sigma(D_{k\gamma})\right|\right|^2_F\left(s-k\gamma\right)&\le 2 \int_{k\gamma}^{t}\frac{\gamma}{m}\left(s-k\gamma\right)\left(\tilde{C}_4+\tilde{C}_1 \frac{\gamma^2}{m}+\tilde{C}_2 \frac{k}{m}\right).
	\end{align*}
	Combining the above terms we get 
	\begin{align*}
	\int_{k\gamma}^{t} \mathbb{E}\left|D_{s}-D_{k\gamma}\right|^2 &\le \tilde K_{11}\gamma^3+\tilde{K}_{12} \frac{\gamma^5}{m}+\tilde{K}_{13}\frac{\gamma^2}{m},
	\end{align*}
	where 
	\begin{align*}
	\tilde{K}_{11}=\left(\frac{2}{3}\tilde{C}_3+\frac{1}{m}\tilde{C}_4\right) ,
	\end{align*}
	and $K_{12}$ can be taken as 
	\begin{align*}
	\tilde{K}_{12}=\frac{5}{3}\tilde{C}_1,
	\end{align*}
	and 	
	\begin{align*}
	\tilde{K}_{13}= \frac{5T}{3}\tilde{C}_2,
	\end{align*}
	Hence, we can write 
	\begin{align*}
	\int_{k\gamma}^{t} \mathbb{E}\left|D_{s}-D_{k\gamma}\right|^2 \le K_{11}\gamma^3 + K_{12}\frac{\gamma^2}{m},
	\end{align*}
	where $K_{11}=\tilde{K}_{11}$, $K_{12}=\tilde{K}_{12}+\tilde{K}_{13}$.
	This completes the proof.
\end{proof}
$ $\linebreak
Next we prove Theorem~\ref{thmggg1}.
\begin{proof}[Proof of Theorem~\ref{thmggg1}]
	We  focus on the interval $(k\gamma,(k+1)\gamma]$, i.e., $t \in (k\gamma,(k+1)\gamma]$. With some abuse of notation we define $j\gamma$ for $j=0,1,2,\cdot,k-1$; with $jk=t$. We use this abuse of notation as this helps with otherwise cumbersome notation.  Using the definition of \eqref{intsgd} and \eqref{sgddiff}, we have
	\begin{align*}
	\left|X_t-D_{t}\right|&\le \left|\int_{0}^{t}\nabla g(X_s)-\int_{0}^{t}\sum_{j=0}^{k-1}\nabla g(D_{j\gamma})I(s)_{[j\gamma,(j+1)\gamma)}\right|ds\\
	&+\sqrt{\frac{\gamma}{m}}\left|\int_{0}^{t}\left(\sigma(X_s)-\sum_{j=0}^{k-1}\sigma(D_{j\gamma})I(s)_{[j\gamma,(j+1)\gamma)}\right)dB_s\right|.
	\end{align*}
	Thus
	\begin{align*}
	&\left|X_t-D_{t}\right|\\
	&\le  \sum_{j=0}^{k-1}\int_{j\gamma}^{(j+1)\gamma}\left|\nabla g(X_s)-\nabla g(D_{j\gamma})\right| ds+\sqrt{\frac{\gamma}{m}}\left|\sum_{j=0}^{k-1} \int_{j\gamma}^{(j+1)\gamma}\left(\sigma(X_s)-\sigma(D_{j\gamma})\right)dB_s\right| \\
	&\le L \sum_{j=0}^{k-1}\int_{j\gamma}^{(j+1)\gamma}\left| X_s-D_{j\gamma}\right|+\sqrt{\frac{\gamma}{m}}\left|\sum_{j=0}^{k-1} \int_{j\gamma}^{(j+1)\gamma}\left(\sigma(X_s)-\sigma(D_{j\gamma})\right)dB_s\right|.
	\end{align*}
	Hence, by using triangle inequality,
	\begin{align*}
	&\left|X_t-D_{t}\right|\\
	&\le L \sum_{j=0}^{k-1}\int_{j\gamma}^{(j+1)\gamma}\left| X_s-D_{s}\right|+L \sum_{j=0}^{k-1}\int_{j\gamma}^{(j+1)\gamma}\left| D_{s} - D_{j\gamma}\right|\\
	& \quad + \sqrt{\frac{\gamma}{m}} \left|\sum_{j=0}^{k-1} \int_{j\gamma}^{(j+1)\gamma}\left(\sigma(X_s)-\sigma(D_{s})\right)dB_s\right|+\left|\sum_{j=0}^{k-1} \int_{j\gamma}^{(j+1)\gamma}\left(\sigma(D_{s})-\sigma(D_{j\gamma})\right)dB_s\right|\\
	&\le L \int_{0}^{t}\left| X_s-D_{s}\right|+ L \sum_{j=0}^{k-1}\int_{j\gamma}^{(j+1)\gamma}\left| D_{s} - D_{j\gamma}\right|+\sqrt{\frac{\gamma}{m}} \left| \int_{0}^{t}\left(\sigma(X_s)-\sigma(D_{s})\right)dB_s\right|\\
	&\quad +\sqrt{\frac{\gamma}{m}} \sum_{j=0}^{k-1} \left| \int_{j\gamma}^{(j+1)\gamma}\left(\sigma(X_s)-\sigma(D_{j\gamma})\right)dB_s\right|.
	\end{align*}
	Now squaring both sides and applying the Cauchy-Schwartz inequality first on all the terms and then the first two integrals (on the second term we apply the Cauchy-Schwartz inequality twice  for the sum and then for the integral), we get 
	\begin{align*}
	&\left|X_t-D_{t}\right|^2\\ &\le 4L^2 t\int_{0}^{t}\left| X_s-D_{s}\right|^2+ 4L^2 k\gamma \sum_{j=0}^{k-1}\int_{j\gamma}^{(j+1)\gamma}\left|D_{s}-D_{j\gamma}\right|^2\\
	&\quad+ 4\frac{\gamma}{m}\left|\int_{0}^{t}\left(\sigma(X_s)-\sigma(D_{s})\right)dB_s\right|^2+4\frac{\gamma}{m}k\sum_{j=0}^{k-1}\left|\int_{j\gamma}^{(j+1)\gamma}\left(\sigma(D_{s})-\sigma(D_{j\gamma})\right)dB_s\right|^2 .	
	\end{align*}
	Taking expectation,
	\begin{align*}
	\mathbb{E}\left|X_t-D_{t}\right|^2 &\le 4L^2 t\int_{0}^{t}\mathbb{E}\left| X_s-D_{s}\right|^2+ 4L^2 k\gamma \sum_{j=0}^{k-1}\int_{j\gamma}^{(j+1)\gamma}\mathbb{E}\left|D_{s}-D_{j\gamma}\right|^2\\
	&\quad+ 4\frac{\gamma}{m} \,\mathbb{E}\left|\int_{0}^{t}\left(\sigma(X_s)-\sigma(D_{s})\right)dB_s\right|^2\\
	&\quad +4\frac{\gamma}{m}k\sum_{j=0}^{k-1}\mathbb{E}\left|\int_{j\gamma}^{(j+1)\gamma}\left(\sigma(D_{s})-\sigma(D_{j\gamma})\right)dB_s\right|^2 .	
	\end{align*}
	Use the  Ito isometry on the last two expressions to give  us
	\begin{align*}
	&\mathbb{E}\left|X_t-D_{t}\right|^2 \\&\le 4L^2 t\int_{0}^{t}\mathbb{E}\left| X_s-D_{s}\right|^2+ 4L^2 k\gamma \sum_{j=0}^{k-1}\int_{j\gamma}^{(j+1)\gamma}\mathbb{E}\left|D_{s}-D_{j\gamma}\right|^2\\
	&\quad+ 4\frac{\gamma}{m}\mathbb{E}\int_{0}^{t}\left|\left|\sigma(X_s)-\sigma(D_{s})\right|\right|_F^2ds +4\frac{k\gamma}{m}\sum_{j=0}^{k-1}\mathbb{E}\int_{j\gamma}^{(j+1)\gamma}\left|\left|\sigma(D_{s})-\sigma(D_{j\gamma})\right|\right|_F^2ds.
	\end{align*}
	Since $\sigma$ is Lipscitz, we obtain
	\begin{align*}
	\mathbb{E}\left|X_t-D_{t}\right|^2
	&\le 4L^2 t\int_{0}^{t}\mathbb{E}\left| X_s-D_{s}\right|^2+ 4L^2 k\gamma \sum_{j=0}^{k-1}\int_{j\gamma}^{(j+1)\gamma}\mathbb{E}\left|D_{s}-D_{j\gamma}\right|^2\\
	&\quad+ 4\frac{\gamma}{m}pL_1^2\int_{0}^{t}\mathbb{E}\left|X_s-D_{s}\right|^2ds +4\frac{k\gamma}{m}pL_1^2 \sum_{j=0}^{k-1} \int_{j\gamma}^{(j+1)\gamma}\mathbb{E}\left|D_{s}-D_{k\gamma}\right|^2ds\\
	&=\left(4L^2t + 4\frac{\gamma}{m}pL_1^2\right)\int_{0}^{t}\mathbb{E}\left| X_s-D_{s}\right|^2\\&\quad \quad \quad +\left(4L^2 k\gamma+4\frac{k\gamma}{m}pL_1^2\right)\sum_{j=0}^{k-1}\int_{j\gamma}^{(j+1)\gamma}\mathbb{E}\left|D_{s}-D_{j\gamma}\right|^2.
	\end{align*}
	Now, from the last step we apply the Gronwall inequality to get
	\begin{align*}
	&\mathbb{E}\left|X_t-D_{t}\right|^2\\ &\le \left[\left(4L^2 k\gamma+4\frac{k\gamma}{m}pL_1^2\right)\sum_{j=0}^{k-1}\int_{j\gamma}^{(j+1)\gamma}\mathbb{E}\left|D_{s}-D_{j\gamma}\right|^2\right]\cdot\exp\left(4L^2\,T^2 + 4\frac{\gamma}{m}pL_1^2 T\right). 
	\end{align*}
	By using the fact $\gamma K=T$, we have
	\begin{align*}
	&\mathbb{E}\left|X_t-D_{t}\right|^2\\ &\le \left[\left(4L^2 T+4\frac{T}{m}pL_1^2\right)\sum_{j=0}^{k-1}\int_{j\gamma}^{(j+1)\gamma}\mathbb{E}\left|D_{s}-D_{j\gamma}\right|^2\right]\cdot\exp\left(4L^2T^2 + 4\frac{\gamma}{m}pL_1^2 T\right).
	\end{align*}
	Invoking Lemma~\ref{middlelemma}, one has
	\begin{align*}
	&\mathbb{E}\left|X_t-D_{t}\right|^2\\ &\le \left[\left(4L^2 T+4\frac{T}{m}pL_1^2\right)\sum_{j=0}^{k-1}\left(K_{11}\gamma^3+K_{12} \frac{\gamma^2}{m}  \right)\right]\cdot\exp\left(4L^2T^2 + 4\frac{\gamma}{m}pL_1^2 T\right)\\
	&\le \left[\left(4L^2 T+4\frac{T}{m}pL_1^2\right)\left(K_{11} T\gamma^2+K_{12}\frac{T\gamma}{m} \right)\right]\cdot\exp\left(4L^2T^2 + 4\frac{\gamma}{m}pL_1^2 T\right)\\
	&= C_{11}\gamma^2+C_{12}\frac{\gamma}{m},
	\end{align*}
	where 
	\begin{align*}
	C_{11}=\left(4L^2 T+4\frac{T}{m}pL_1^2\right)\exp\left(4L^2T^2 + 4\frac{\gamma}{m}pL_1^2 T\right)T \ K_{11}
	\end{align*}
	and 
	\begin{align*}
	C_{12}=\left(4L^2 T+4\frac{T}{m}pL_1^2\right)\exp\left(4L^2T^2 + 4\frac{\gamma}{m}pL_1^2 T\right)T \ K_{12}.
	\end{align*}
	Hence 
	\begin{align*}
	W^2_2(X_t,D_{t})\le C_{11}\gamma^2+C_{12}\frac{\gamma}{m}.
	\end{align*}
	We complete the proof.
\end{proof}
\medskip
We now prove another lemma before proceeding to one of the main theorems.
\begin{lemma}\label{lemmaggg1}
	Under  assumptions~\ref{assm1}-\ref{assm5}, we have
	\begin{align*}
	&\mathbb{E}\left|\sum_{i=1}^{n}w_{i,k}\nabla l(Y_{n,k\gamma},u_{i,k})-\sum_{i=1}^{n}w_{i,k}\nabla l(D_{k\gamma},u_{i,k})\right|^2\\ &\le \left[\frac{1}{m}\left(\mathbb{E}\left(h_1^2(U)\right)-L^2\right)+L^2\right]\mathbb{E}\left|Y_{n,k\gamma}-D_{k\gamma}\right|^2 .
	\end{align*}
\end{lemma}
\begin{proof}[Proof of Lemma~\ref{lemmaggg1}]
	Again we start by bounding the $L_2$ norm
	\begin{align*}
	& \mathbb{E}\left|\sum_{i=1}^{n}w_{i,k}\nabla l(Y_{n,k\gamma},u_{i,k})-\sum_{i=1}^{n}w_{i,k}\nabla l(D_{k\gamma},u_{i,k})\right|^2\\
	&=\mathbb{E}\sum_{i=1}^{n}w^2_{i,k}\left|\nabla l(Y_{n,k\gamma},u_{i,k})-l(D_{k\gamma},u_{i,k})\right|^2\\
	&\quad + \,\mathbb{E}\sum_{i,j,i\ne j} w_{i,k}w_{j,k}\left(\nabla l(Y_{n,k\gamma},u_{i,k})-\nabla l(D_{k\gamma},u_{i,k})\right)^{\mathsf{T}}\\
	&\quad \quad \quad \quad \quad \quad \quad \quad \quad \left(\nabla l(Y_{n,k\gamma},u_{j,k})-\nabla l(D_{k\gamma},u_{j,k})\right) .
	\end{align*}
	Hence 
	\begin{align*}
	& \mathbb{E}\left|\sum_{i=1}^{n}w_{i,k}\nabla l(Y_{n,k\gamma},u_{i,k})-\sum_{i=1}^{n}w_{i,k}\nabla l(D_{k\gamma},u_{i,k})\right|^2\\ 
	& \quad \le \frac{1}{m} \ \mathbb{E}\left(h_1^2(u)\right)\mathbb{E}\left|Y_{n,k\gamma}-D_{k\gamma}\right|^2\\
	&\quad \quad + \sum_{i,j,i\ne j} \frac{m-1}{mn(n-1)} \,\mathbb{E}\left|\nabla g(Y_{n,k\gamma})-\nabla g(D_{k\gamma})\right|^2\\
	&\quad \le \frac{1}{m} \ \mathbb{E}\left(h_1^2(u)\right)\mathbb{E}\left|Y_{n,k\gamma}-D_{k\gamma}\right|^2 \\
	&\quad \quad + L^2\left(1-\frac{1}{m}\right)L^2 \mathbb{E}\left|Y_{n,k\gamma}-D_{k\gamma}\right|^2 . 
	\end{align*}
	This completes the proof.
\end{proof}

Next we exhibit that the interpolated M-SGD process and the interpolated SGD with scaled normal error are close.
Here $t \in (k\gamma,(k+1)\gamma]$.
\begin{prop}\label{theoremggg2}
	Under assumptions~\ref{assm1}-\ref{assm5}, for $t\in (k\gamma,(k+1)\gamma]$, we have 
	\begin{align*}
	W_2^2(Y_{n,t},D_{t})\le \tilde{J}_1\frac{3^k\gamma}{m}
	\end{align*}
	where $\tilde{J}_1$ is a constant dependent on $T,L,L_1,p, \,\mathbb{E}\left(h_1^2(u)\right)$.
\end{prop}
\begin{proof}[Proof of Proposition~\ref{theoremggg2}]
	Using the definitions of~\eqref{intsgd} and~\eqref{intmsgd}, we have 
	\begin{align*}
	\left|Y_{n,t}-D_{t}\right|&\le \left|Y_{n,k\gamma}-D_{k\gamma}\right| \ +\left(t-k\gamma\right)\left|\sum_{i=1}^{n}w_{i,k}\nabla l(Y_{n,k\gamma},u_{i,k})-\nabla g(D_{k\gamma})\right|\\
	&\quad +\sqrt{\frac{\gamma}{m}}\left|\sigma(D_{k\gamma})\left(B_t-B_{k\gamma}\right)\right|\\
	&\le \left|Y_{n,k\gamma}-D_{k\gamma}\right| \ +\left(t-k\gamma\right)\left|\sum_{i=1}^{n}w_{i,k}\left(\nabla l(Y_{n,k\gamma},u_{i,k})-\nabla l(D_{k\gamma},u_{i,k})\right)\right|\\
	&\quad +\left(t-k\gamma\right)\left|\sum_{i=1}^{n}w_{i,k}\nabla l(D_{k\gamma},u_{i,k})-\nabla g(D_{k\gamma})\right| \ \\&\quad \quad \quad +\sqrt{\frac{\gamma}{m}}\left|\sigma(D_{k\gamma})\left(B_t-B_{k\gamma}\right)\right|.
	\end{align*}
	Thus, we get
	\begin{align*}
	&\left|Y_{n,t}-D_{t}\right|^2\\ &\le 4 \left|Y_{n,k\gamma}-D_{k\gamma}\right|^2 \ +4\left(t-k\gamma\right)^2\left|\sum_{i=1}^{n}w_{i,k}\left(\nabla l(Y_{n,k\gamma},u_{i,k})-\nabla l(D_{k\gamma},u_{i,k})\right)\right|^2\\
	&\quad +4\left(t-k\gamma\right)^2\left|\sum_{i=1}^{n}w_{i,k}\nabla l(D_{k\gamma},u_{i,k})-\nabla g(D_{k\gamma})\right|^2 \ +4\frac{\gamma}{m}\left|\sigma(D_{k\gamma})\left(B_t-B_{k\gamma}\right)\right|^2.
	\end{align*}
	Taking expectation, we see
	\begin{align*}
	&\mathbb{E}\left|Y_{n,t}-D_{t}\right|^2\\ &\le 4 \, \mathbb{E}\left|Y_{n,k\gamma}-D_{k\gamma}\right|^2 \ +4\left(t-k\gamma\right)^2\mathbb{E}\left|\sum_{i=1}^{n}w_{i,k}\left(\nabla l(Y_{n,k\gamma},u_{i,k})-\nabla l(D_{k\gamma},u_{i,k})\right)\right|^2\\
	&\quad +4\left(t-k\gamma\right)^2\mathbb{E}\left|\sum_{i=1}^{n}w_{i,k}\nabla l(D_{k\gamma},u_{i,k})-\nabla g(D_{k\gamma})\right|^2 \ +4\frac{\gamma}{m}\mathbb{E}\left|\sigma(D_{k\gamma})\left(B_t-B_{k\gamma}\right)\right|^2.
	\end{align*}
	This implies 
	\begin{align*}
	&\mathbb{E}\left|Y_{n,t}-D_{t}\right|^2\\
	&\le 4 \,\mathbb{E}\left|Y_{n,k\gamma}-D_{k\gamma}\right|^2 \ +4\left(t-k\gamma\right)^2\left(\frac{1}{m}\left(\mathbb{E}\left(h_1^2(u)\right)-L^2\right)+L^2\right)\mathbb{E}\left|Y_{n,k\gamma}-D_{k\gamma}\right|^2\\
	&\quad +4\left(t-k\gamma\right)^2\frac{1}{m} \,\mathbb{E}\left|\left|\sigma(D_{k\gamma})\right|\right|^2_F\ +4\frac{\gamma}{m}\left(t-k\gamma\right)\mathbb{E}\left|\left|\sigma(D_{k\gamma})\right|\right|^2_F.	
	\end{align*}
	The last line follows using Lemma~\ref{lemmaggg1} and the Ito isometry. Therefore
	\begin{align*}
	\mathbb{E}\left|Y_{n,t}-D_{t}\right|^2&\le \left[4+4\left(t-k\gamma\right)\left(\frac{1}{m}\left(\mathbb{E}\left(h_1^2(u)\right)-L^2\right)+L^2\right)\right]\mathbb{E}\left|x_{n,k}-x_k\right|^2\\
	&\quad+2\left[\frac{1}{m}4\left(t-k\gamma\right)^2+4\frac{\gamma}{m}\left(t-k\gamma\right)\right]\left(pL^2_1 \, \mathbb{E}\left|x_k-\tilde{x}_k\right|^2+\left|\left|\sigma(\tilde{x}_k)\right|\right|^2_F\right).
	\end{align*}
	Here we use the fact that $\mathbb{E}\left|Y_{n,k\gamma}-D_{k\gamma}\right|^2= \,\mathbb{E}\left|x_{n,k}-x_k\right|^2$. We also use the facts $\mathbb{E}\left|D_{k\gamma}-\tilde{x}_k\right|^2= \,\mathbb{E}\left|x_{k}-\tilde{x}_k\right|^2$ and  $\sigma(\cdot)$ is Lipschitz. Using this, along with Proposition~\ref{thmgg1}, Lemma~\ref{moderatediffsq} and Lemma~\ref{gsigmabound}, we obtain,
	\begin{align*}
	\mathbb{E}\left|Y_{n,t}-D_{t}\right|^2 \le J_{1}+J_{2}
	\end{align*}
	where
	\begin{align*}
	J_{1}&=\left\{4+4\left(t-k\gamma\right)\left(\frac{1}{m}\left[\mathbb{E}\left(h_1^2(u)\right)-L^2\right]+L^2\right)\right\}\\
	&\quad \quad \cdot K_1 \frac{3^k\gamma}{m},
	\end{align*}
	and 
	\begin{align*}
	J_2&=2\left[\frac{1}{m}4\left(t-k\gamma\right)^2+4\frac{\gamma}{m}\left(t-k\gamma\right)\right]\\
	&\quad \quad \cdot \Bigg\{\tilde{C}_1\frac{\gamma^2}{m}+\tilde{C}_2\frac{k}{m}+ \left[2C_1^2 k^2 \gamma^2L_1^2p\left(1+L\gamma\right)^{2k}+2\sup_{0\le t\le T}\left|\left|\sigma(\tilde{X}_t)\right|\right|^2_F\right]\Bigg\}\\
	&\le2\left[\frac{1}{m}4\gamma^2+4\frac{\gamma}{m}\left(t-k\gamma\right)\right]\\
	&\quad \quad \cdot \Bigg\{\tilde{C}_1\frac{\gamma^2}{m}+\tilde{C}_2\frac{k}{m}+ \left[2C_1^2 T^2 L_1^2p e^{2TL}+2\sup_{0\le t\le T}\left|\left|\sigma(\tilde{X}_t)\right|\right|^2_F\right]\Bigg\}.
	\end{align*}
	Note that, 
	$J_1 \le \frac{3^k\gamma}{m} J_{11}$, where
	\begin{align*}
	J_{11}&= K_1 \left[4+4\,\mathbb{E}\left(h_1^2(u)\right)\right].
	\end{align*} 
	Also
	\begin{align*}
	J_2=J_{21} \frac{\gamma^4}{m^2}+J_{22}\frac{\gamma}{m^2}+J_{23}\frac{\gamma^2}{m}
	\end{align*}
	where $J_{21},J_{22},J_{23}$ can be chosen as 
	\begin{align*}
	J_{21}&=8 \tilde{C}_1,\\
	J_{22}&=8 \tilde{C}_2 T,\\
	J_{23}&=8 \left[2C_1^2 T^2 L_1^2p e^{2TL}+2\sup_{0\le t\le T}\left|\left|\sigma(\tilde{X}_t)\right|\right|^2_F\right].
	\end{align*}
	This concludes
	\begin{align*}
	E\left|Y_{n,t}-D_{t}\right|^2&\le \frac{3^k\gamma}{m} J_{11}+J_{21} \frac{\gamma^4}{m^2}+J_{22}\frac{\gamma}{m^2}+J_{23}\frac{\gamma^2}{m} \\
	&\le \tilde J_{1} \frac{3^k\gamma}{m},
	\end{align*}
	where we can consider $\tilde{J}_1$ as
	\begin{align*}
	\tilde{J}_1 = J_{11}+J_{21}+J_{22}+J_{23}.
	\end{align*}
	Therefore it follows that
	\begin{align*}
	W_2^2(Y_{n,t},D_{t})\le \tilde{J}_1\frac{3^k\gamma}{m}.
	\end{align*}
	The proof is finished.
\end{proof}
Next we prove one of our main theorems.  

\begin{proof}[Proof of Theorem~\ref{mainthmgg1} ]
	Let $t \in [k\gamma,(k+1)\gamma)$. Using the fact that the Wasserstein distance exhibits the inequality $W^2_2(\mu,\nu)\le 2 W^2_2(\mu,P)+2W^2_2(P,\nu)$ (where $\mu,\nu,P$ are probability measures), we have
	\begin{align*}
	W_2^2(Y_{n,t},X_t)& \le 2W_2^2(Y_{n,t},D_{t})+2W_2^2(D_{t},X_t)\\
	&\le 2\tilde{J}_1\frac{3^k\gamma}{m}+2C_{11}\gamma^2+2C_{12}\frac{\gamma}{m}\\
	&\le C_{21}\gamma^2+C_{22}\frac{3^k\gamma}{m},
	\end{align*}
	where $C_{21}=2C_{11}$ and $C_{22}=2\tilde{J}_1+C_{12}$.
	Hence we conclude the proof.
\end{proof}
\begin{prop}\label{theoremggg3}
	Under assumptions~\ref{assm1}-\ref{assm5}, with $w_{ik} \ge 0$, for $t\in (k\gamma,(k+1)\gamma]$, we have 
	\begin{align*}
	W_2^2(Y_{n,t},D_{t})\le \tilde{J}_{11}\frac{\gamma^2}{m}+\tilde{J}_{12}\frac{\gamma}{m^2}+\tilde{J}_{13}\frac{k}{m}
	\end{align*}
	where $\tilde{J}_{11}, J_{12}$ are constants dependent on $T,L,L_1,p, \,\mathbb{E}\left(h_1^2(u)\right)$.
\end{prop}
\begin{proof}[Proof of Proposition~\ref{theoremggg3}]
    From Proposition~\ref{theoremggg2}, we know that 
    	\begin{align*}
	\mathbb{E}\left|Y_{n,t}-D_{t}\right|^2&\le \left[4+4\left(t-k\gamma\right)\left(\frac{1}{m}\left(\mathbb{E}\left(h_1^2(u)\right)-L^2\right)+L^2\right)\right]\mathbb{E}\left|x_{n,k}-x_k\right|^2\\
	&\quad+2\left[\frac{1}{m}4\left(t-k\gamma\right)^2+4\frac{\gamma}{m}\left(t-k\gamma\right)\right]\left(pL^2_1 \, \mathbb{E}\left|x_k-\tilde{x}_k\right|^2+\left|\left|\sigma(\tilde{x}_k)\right|\right|^2_F\right).
	\end{align*}
	For the first term we get 
	\begin{align*}
&\left[4+4\left(t-k\gamma\right)\left(\frac{1}{m}\left(\mathbb{E}\left(h_1^2(u)\right)-L^2\right)+L^2\right)\right]\mathbb{E}\left|x_{n,k}-x_k\right|^2\\
&\le 4\,\left(1+\mathbb{E}(h^2_1(u))\right)\left(K_{11}\frac{\gamma^2}{m}+K_{12}\frac{k}{m}\right).
	\end{align*}
For the second term, we get 
\begin{align*}
&2\left[\frac{1}{m}4\left(t-k\gamma\right)^2+4\frac{\gamma}{m}\left(t-k\gamma\right)\right]\left(pL^2_1 \, \mathbb{E}\left|x_k-\tilde{x}_k\right|^2+\left|\left|\sigma(\tilde{x}_k)\right|\right|^2_F\right)\\
&\le 16\, \frac{\gamma^2}{m}\, \left[p\, L^2_1\left(C^*_1\frac{\gamma^2}{m}+C^*_2\frac{k}{m}\right)+\left\|\sigma(\tilde{x}_k)\right\|^2_F\right]\\
&=16\, p\, L^2_1\, C^*_1 \frac{\gamma^4}{m^2}+16\, p\, L^2_1\, C^*_2\, T \frac{\gamma}{m^2}+16\, \left\|\sigma(\tilde{x}_k)\right\|^2_F\frac{\gamma^2}{m}.
\end{align*}
Therefore we have 
\begin{align*}
 	&\mathbb{E}\left|Y_{n,t}-D_{t}\right|^2  \\
 	& \le \left[4\, K_{11}\left(1+\mathbb{E}(h^2_1(u))\right)+16\, p\, L^2_1\,C^*_1+16\, \left\|\sigma(\tilde{x}_k)\right\|^2_F\right]\frac{\gamma^2}{m}+16\, p\, L^2_1\, C^*_2\, T \frac{\gamma}{m^2}+4\, K_{12}\left(1+\mathbb{E}(h^2_1(u))\right)\frac{k}{m}.
\end{align*}
Hence we are done with 
\begin{align*}
    \tilde{J}_{11}&=\left[4\, K_{11}\left(1+\mathbb{E}(h^2_1(u))\right)+16\, p\, L^2_1\,C^*_1+16\, \left\|\sigma(\tilde{x}_k)\right\|^2_F\right]\\
    \tilde{J}_{12}&=16\, p\, L^2_1\, C^*_2\, T\\
    \tilde{J}_{13}&=4\, K_{12}\left(1+\mathbb{E}(h^2_1(u))\right).
\end{align*}
\end{proof}

\begin{proof}[Proof of Theorem~\ref{mainthmgg2}]
We have
	\begin{align*}
	W_2^2(Y_{n,t},X_t)& \le 2W_2^2(Y_{n,t},D_{t})+2W_2^2(D_{t},X_t)\\
	&\le \tilde{J}_{11}\frac{\gamma^2}{m}+\tilde{J}_{12}\frac{\gamma}{m^2}+\tilde{J}_{13}\frac{k}{m}+2C_{11}\gamma^2+2C_{12}\frac{\gamma}{m}\\
	&\le C_{23}\gamma^2+C_{24}\frac{\gamma}{m}.
	\end{align*}
	Here 
	\begin{align*}
	    C_{23}&=2\,C_{11}+\tilde{J}_{11}\\
	    C_{24}&=\tilde{J}_{12}+\frac{1}{T}\,\tilde{J}_{13}+2\,C_{12}.
	\end{align*}
\end{proof}
%%%%{Next section for strongly convex functions}%%%%%
\subsection*{Proofs for the Convex Regime}
In the case of the objective function $g$ being strongly convex, we derive bounds for the M-SGD algorithm with the structure as mentioned previously.

We consider the algorithm~\ref{algo4}. Under assumptions~\ref{assm1}-\ref{assm5} and~\ref{assm7}, we exhibit that the algorithm converges to the global optimum on average.

\begin{proof}[Proof of Proposition~\ref{optstrgcnvx}]
	We know
	\begin{align*}
	x_{n,k+1}&=x_{n,k}-\gamma \sum_{i=1}^{n}w_{i,k}\nabla l(x_{n,k},u_{i,k})\\
	&=x_{n,k}-\gamma \nabla g(x_{n,k})-\gamma \sum_{i=1}^{n}w_{i,k}\left(\nabla l(x_{n,k},u_{i,k})-\nabla g(x_{n,k})\right).
	\end{align*}
	Hence, we have 
	\begin{align*}
	g(x_{n,k+1})&=g(x_{n,k}-\gamma \nabla g(x_{n,k})-\gamma \sum_{i=1}^{n}w_{i,k}(\nabla l(x_{n,k},u_{i,k})-\nabla g(x_{n,k}))).\\
	\end{align*}
	Therefore 
	\begin{align*}
	g(x_{n,k+1})&=g(x_{n,k})-\gamma \nabla g(x_{n,k})^{\mathsf{T}}\left(\nabla g(x_{n,k})+\sum_{i=1}^{n}w_{i,k}\left(\nabla l(x_{n,k},u_{i,k})-\nabla g(x_{n,k})\right)\right)\\
	&\quad +\frac{\gamma^2}{2}\left(\nabla g(x_{n,k})+\sum_{i=1}^{n}w_{i,k}\left(\nabla l(x_{n,k},u_{i,k})-\nabla g(x_{n,k})\right)\right)^{\mathsf{T}}\nabla^2g(\hat{x}_{n,k})\\
	&\quad \quad \quad \ \left(\nabla g(x_{n,k})+\sum_{i=1}^{n}w_{i,k}\left(\nabla l(x_{n,k},u_{i,k})-\nabla g(x_{n,k})\right)\right).
	\end{align*}
	Thus, 
	\begin{align*}
	&g(x_{n,k+1})	\\&\le g(x_{n,k})-\gamma |\nabla g(x_{n,k})|^2-\gamma \nabla g(x_{n,k})^{\mathsf{T}}\left(\sum_{i=1}^{n}w_{i,k}\left(\nabla l(x_{n,k},u_{i,k})-\nabla g(x_{n,k})\right)\right)\\
	&\quad +\frac{L\gamma^2}{2}\left|\nabla g(x_{n,k})+ \sum_{i=1}^{n}w_{i,k}\left(\nabla l(x_{n,k},u_{i,k})-\nabla g(x_{n,k})\right)\right|^2.
	\end{align*}
	The last line follows from the fact that $\nabla g$ is Lipschitz. By taking expectation, we get
	\begin{align*}
	&\mathbb{E}\left( g(x_{n,k+1})\right)\\& \le \mathbb{E} \left(g(x_{n,k})\right)-\gamma \,\mathbb{E}\left|\nabla g(x_{n,k})\right|^2+\frac{L\gamma^2}{2} \,\mathbb{E}\left|\nabla g(x_{n,k})\right|^2+\frac{L\gamma^2}{2m} \,\mathbb{E} \left(\Tr \sigma^2(x_{n,k})\right) \\
	&\le  \mathbb{E}\left(g(x_{n,k})\right)-\gamma \left(1-\frac{L\gamma}{2}\right) \mathbb{E}\left|\nabla g(x_{n,k})\right|^2+\frac{L\gamma^2}{2m} \,\mathbb{E} \left(\Tr \sigma^2(x_{n,k})\right) .
	\end{align*}
	The second line follows as this is the online version of the algorithm, i.e., we refresh the $u_i$ at each iteration and $\nabla l(\cdot,u)$ is unbiased for $g(\cdot)$. Also we use the definition that $\sigma(\cdot)$ is the variance of $\nabla l(\cdot,u)$. 
	From the works of Boyd and Vandenburghe~\cite{boyd2004convex}, we have $$\left|\nabla g(x)\right|^2 \ge 2\lambda \left(g(x)-g(x^*)\right), \ \forall x .$$
	From the previous inequality, it follows that
	\begin{align*}
	\mathbb{E} \left(g(x_{n,k+1})\right) &\le \mathbb{E} \left(g(x_{n,k})\right)-\gamma \lambda  \left(2-L\gamma\right) \mathbb{E}\left(g(x_{n,k})-g(x^*)\right)+\frac{L\gamma^2}{2m} \,\mathbb{E} \left|\left|\sigma(x_{n,k})\right|\right|^2_F\\
	&\le \mathbb{E} \left(g(x_{n,k})\right)-\gamma \lambda  \left(2-L\gamma\right) \mathbb{E}\left(g(x_{n,k})-g(x^*)\right)+\frac{L\gamma^2}{2m}\{2 \, \mathbb{E} \left|\left|\sigma(x^*)\right|\right|^2_F\\
	&\quad \quad \quad  + 2pL^2_1 \,\mathbb{E}\left|x_{n,k}-x^*\right|^2\}.
	\end{align*} 
	The last line uses the fact that $\sigma(\cdot)$ is $\sqrt p L_1$ Lipschitz in the Frobenius norm which follows from the fact that it is $L_1$ Lipschitz in the spectral norm. Using the fact that $g$ is $\lambda$-strongly convex, one has $$g(y)-g(x)\ge \nabla g(x)^{\mathsf{T}}(y-x)+\frac{\lambda}{2}|y-x|^2.$$
	Replacing $y=x_k$ and $x=x^*$ and using the fact that $\nabla g(x^*)=0$, we have $g(x_k)-g(x^*)\ge \frac{\lambda}{2}|x_k-x^{*}|^2$. Thus from the final line, subtracting $g(x^*)$ to both sides, we get,
	\begin{align*}
	&\mathbb{E} \left(g(x_{n,k+1})-g(x^*)\right)\\ &\le \mathbb{E} \left(g(x_{n,k})-g(x^*)\right)-\lambda \gamma (2-L\gamma) \,\mathbb{E} \left(g(x_{n,k})-g(x^*)\right)\\
	&\quad +\frac{2pLL^2_1\gamma^2}{m\lambda} \,\mathbb{E} \left(g(x_{n,k})-g(x^*)\right)+\frac{L\gamma^2}{m}\left|\left|\sigma(x^*)\right|\right|^2_F\\
	&=\left[1-\lambda\gamma(2-L\gamma)+\frac{2pLL^2_1\gamma^2}{m\lambda}\right] \mathbb{E} \left(g(x_{n,k})-g(x^*)\right)+\frac{L\gamma^2}{m}\left|\left|\sigma(x^*)\right|\right|^2_F.
	\end{align*}
	Note that due to our final assumption on $\gamma, \ m $, $\left[1-\lambda\gamma(2-L\gamma)+\frac{2pLL^2_1\gamma^2}{m\lambda}\right]<1$. Calling this quantity as $r$, $\mathbb{E} \left(g(x_{n,k})-g(x^*)\right)=a_k$ and $\left|\left|\sigma(x^*)\right|\right|^2_F=B$, we have the last line as 
	\begin{align*}
	a_{k+1}&\le r a_k+\frac{L\gamma^2}{m}B\\
	& \le r^{k+1} a_0+\frac{L\gamma^2}{m}B\left(1+r+r^2+\cdots+r^k\right)\\
	&\le r^{k+1} a_0+\frac{L\gamma^2}{m(1-r)}B.
	\end{align*}
	Therefore the first result for \eqref{msgd} follows from the last line. The second result for \eqref{msgd} follows using strong convexity of $g$ which implies that $g(y)-g(x^*)\ge \frac{\lambda}{2}|y-x^*|^2$. The proof for \eqref{sgd} is exactly same and hence we skip it.
\end{proof}
\begin{proof}[Proof of Theorem~\ref{wass:conv:stng:convx}]
    From the proof of proposition~\ref{theoremggg2}, we know that 
    	\begin{align*}
	\mathbb{E}\left|Y_{n,t}-D_{t}\right|^2&\le \underset{I*}{\underbrace{\left[4+4\left(t-k\gamma\right)\left(\frac{1}{m}\left(\mathbb{E}\left(h_1^2(u)\right)-L^2\right)+L^2\right)\right]\mathbb{E}\left|x_{n,k}-x_k\right|^2}}\\
	&\quad+\underset{II*}{\underbrace{2\left[\frac{1}{m}4\left(t-k\gamma\right)^2+4\frac{\gamma}{m}\left(t-k\gamma\right)\right]\left(pL^2_1 \, \mathbb{E}\left|x_k-\tilde{x}_k\right|^2+\left|\left|\sigma(\tilde{x}_k)\right|\right|^2_F\right)}}.
	\end{align*}
 Note that 
 \begin{align*}
     \mathbb{E}\left|x_{n,k}-x_k\right|^2&\le 2\left(\mathbb{E}\left|x_{n,k}-x^*\right|^2+\mathbb{E}\left|x^*-x_k\right|^2\right)\\
     &\le \frac{8}{\lambda}\left[1-\lambda\gamma(2-L\gamma)+\frac{2pLL^2_1\gamma^2}{m\lambda}\right]^{k} \left(g(x_{0})-g(x^*)\right)\\
	&\quad +\frac{8}{\lambda}\left[\frac{L\gamma}{m\left(\lambda(2-L\gamma)-\frac{2pLL^2_1\gamma}{m\lambda}\right)}\right] \left|\left|\sigma(x^*)\right|\right|^2_F.
 \end{align*}
Also,
\begin{align*}
    \mathbb{E}\left|x_k-\tilde{x}_k\right|^2 &\le 2\left(\mathbb{E}\left|\tilde{x}_k-x^*\right|^2+\mathbb{E}\left|x^*-x_k\right|^2\right)\\
    &\le \left[1-\gamma \lambda\left(2-L\gamma\right) \right]^k \left(g(x_0)-g(x^*)\right)+\frac{4}{\lambda}\left[1-\lambda\gamma(2-L\gamma)+\frac{2pLL^2_1\gamma^2}{m\lambda}\right]^{k} \left(g(x_{0})-g(x^*)\right)\\
	&\quad +\frac{4}{\lambda}\left[\frac{L\gamma}{m\left(\lambda(2-L\gamma)-\frac{2pLL^2_1\gamma}{m\lambda}\right)}\right] \left|\left|\sigma(x^*)\right|\right|^2_F.
\end{align*}
Hence for the first term  
\begin{align*}
    I^* &\le \frac{32}{\lambda}\,\left(1+\mathbb{E}\left(h_1^2(u)\right)\right)\, \left(g(x_0)-g(x^*)\right)\left[1-\lambda\gamma(2-L\gamma)+\frac{2pLL^2_1\gamma^2}{m\lambda}\right]^{k}\\
    &\quad \quad +\frac{16}{\lambda^2\, m}\, L\, \left\|\sigma(x^*)\right\|^2_F\, \gamma
\end{align*}
and for the second term
\begin{align*}
    II^* &\le 8\, p\, L^2_1\left[\left(1+\frac{4}{\lambda}\right)\left(g(x_0)-g(x^*)\right)+\frac{4\,L}{\lambda}\left\|\sigma(x^*)\right\|^2_F\right]\gamma^2\\
    &\quad \quad +16\left(g(x_0)-g(x^*)+\left\|\sigma(x^*)\right\|^2_F \right)\gamma^2.
\end{align*}
Therefore we have 
\begin{align*}
    W^2_2(Y_{n,t},D_{n,t})\le C^{**}_1\rho^k+C^{**}_2\gamma^2+C^{**}_3\gamma,
\end{align*}
where
\begin{align*}
    &\rho=\left[1-\lambda\gamma(2-L\gamma)+\frac{2pLL^2_1\gamma^2}{m\lambda}\right],\\
  &C^{**}_1=\frac{32}{\lambda}\,\left(1+\mathbb{E}\left(h_1^2(u)\right)\right)\, \left(g(x_0)-g(x^*)\right), \\
  &C^{**}_2= 8\, p\, L^2_1\left[\left(1+\frac{4}{\lambda}\right)\left(g(x_0)-g(x^*)\right)+\frac{4\,L}{\lambda}\left\|\sigma(x^*)\right\|^2_F\right],\\
    &\quad \quad +16\left(g(x_0)-g(x^*)+\left\|\sigma(x^*)\right\|^2_F \right),\\
  &C^{**}_3=\frac{16}{\lambda^2\, m}\, L\, \left\|\sigma(x^*)\right\|^2_F.
\end{align*}
Using the fact that 
	\begin{align*}
	W_2^2(Y_{n,t},X_t)& \le 2W_2^2(Y_{n,t},D_{t})+2W_2^2(D_{t},X_t)
	\end{align*}
we have 
	\begin{align*}
	W_2^2(Y_{n,t},X_t)
	&\le C^{**}_1\rho^k+C^{**}_2\gamma^2+C^{**}_3\gamma+2C_{11}\gamma^2+2C_{12}\frac{\gamma}{m}\\
 &\le \tilde{C}^{**}_1\rho^k+\tilde{C}^{**}_2\gamma^2+\tilde{C}^{**}_3\gamma
	\end{align*}
where 
\begin{align*}
   &\tilde{C}^{**}_1=C^{**}_1,\\
   & \tilde{C}^{**}_2=C^{**}_2+2\,C_{11},\\
   & \tilde{C}^{**}_3=C^{**}_3+2\, C_{12}.
\end{align*}
\end{proof}
\end{document}